\documentclass[10pt,twocolumn,letterpaper]{article}

\usepackage[pagenumbers]{wacv} 

\usepackage{graphicx}
\usepackage{amsmath}
\usepackage{amssymb}
\usepackage{booktabs}
\usepackage{adjustbox}
\usepackage{subcaption} 
\usepackage{multirow}
\usepackage{cite}
\usepackage{placeins}
\usepackage{amsthm}
\usepackage[linesnumbered,ruled,vlined]{algorithm2e}
\DeclareMathOperator*{\argmin}{arg\,min} 
\newtheorem{theorem}{Theorem}
\newtheorem{proposition}{Proposition}

\newtheorem{extension}{Extention}
\usepackage[pagebackref,breaklinks,colorlinks]{hyperref}

\usepackage[capitalize]{cleveref}
\crefname{section}{Sec.}{Secs.}
\Crefname{section}{Section}{Sections}
\Crefname{table}{Table}{Tables}
\crefname{table}{Tab.}{Tabs.}

\def\wacvPaperID{2414} 
\def\confName{WACV}
\def\confYear{2024}

\begin{document}

\title{Learning Low-Rank Latent Spaces with Simple Deterministic Autoencoder: Theoretical and Empirical Insights}

\author{Alokendu Mazumder\\
Indian Institute of Science\\
Bengaluru, India\\
{\tt\small alokendum@iisc.ac.in}
\and
Tirthajit Baruah\textsuperscript{\textdagger}\\
Indian Institute of Science\\
Bengaluru, India\\
{\tt\small tirthajitb@iisc.ac.in}
\and
Bhartendu Kumar\thanks{Work done during his MTech at IISc Bengaluru} \thanks{denotes equal second place/author contribution}\\
TCS Research\\
Bengaluru, India\\
{\tt\small k.bhartendu@tcs.com}
\and
Rishab Sharma\thanks{Work done during his internship at IISc Bengaluru}\\
Dayananda Sagar College of Engineering\\
Bengaluru, India\\
{\tt\small trishab2001rs@gmail.com}
\and
Vishwajeet Pattanaik\\
Indian Institute of Science\\
Bengaluru, India\\
{\tt\small vishwajeetp@iisc.ac.in}
\and
Punit Rathore\\
Indian Institute of Science\\
Bengaluru, India\\
{\tt\small prathore@iisc.ac.in}
}
\maketitle

\begin{abstract}
    The autoencoder is an unsupervised learning paradigm that aims to create a compact latent representation of data by minimizing the reconstruction loss. However, it tends to overlook the fact that most data (images) are embedded in a lower-dimensional space, which is crucial for effective data representation. To address this limitation, we propose a novel approach called Low-Rank Autoencoder (LoRAE). In LoRAE, we incorporated a low-rank regularizer to adaptively reconstruct a low-dimensional latent space while preserving the basic objective of an autoencoder. 
    This helps embed the data in a lower-dimensional space while preserving important information. It is a simple autoencoder extension that learns low-rank latent space. Theoretically, we establish a tighter error bound for our model. Empirically, our model's superiority shines through various tasks such as image generation and downstream classification. Both theoretical and practical outcomes highlight the importance of acquiring low-dimensional embeddings.

\end{abstract}
\section{Introduction}
\label{sec:intro}
Learning effective representations remains a fundamental challenge in the field of artificial intelligence~\cite{bengio2013representation}. These representations, acquired through self-supervised or unsupervised learning, serve as valuable foundations for various downstream tasks like generation and classification. Among the methods employed for unsupervised representation learning, \textit{autoencoders} (AEs) have gained popularity. AEs allow the extraction of meaningful features from data without the need for labelled examples. The process involves the transformation of data into a lower-dimensional space recognized as the latent space. Subsequently, this transformed data is reconstructed to match its original form, facilitating the acquisition of a meaningful and condensed representation known as a latent vector. To ensure the avoidance of trivial identity mapping, a key aspect is to restrict the information capacity within the autoencoder's internal representation. Over the course of time, various adaptations of autoencoders have been introduced to address this limitation. The Diabolo network~\cite{rumelhart1985learning}, for instance, simplifies the approach by employing a low-dimensional representation. On the other hand, \emph{variational autoencoders} (VAEs)~\cite{kingma2013auto} introduce controlled noise into latent vectors while constraining the distribution's variance. Denoising Autoencoders~\cite{vincent2008extracting} are trained to intentionally generate substantial reconstruction errors by introducing random noise to their inputs. Meanwhile, Sparse Autoencoders~\cite{ranzato2007sparse} enforce a strict sparsity penalty on latent vectors. Quantized Autoencoders like \textit{vector quantized} VAE (VQ-VAE)~\cite{van2017neural} discretize codes into distinct clusters, while Contrasting Autoencoders~\cite{rifai2011contractive} minimize network function curvature beyond the manifold's boundaries. Notably, Low-rank Autoencoders, like \textit{implicit rank minimizing autoencoder} (IRMAE)~\cite{jing2020implicit}, implicitly minimize the rank of the empirical covariance matrix of the latent space by leveraging the dynamics of \textit{stochastic gradient descent} (SGD). In essence, the array of techniques for autoencoder refinement has expanded significantly, leading to a comprehensive toolkit of strategies for overcoming inherent drawbacks and enhancing the efficacy of these latent space learning models.

In this work, we let the network learn the best possible (low) rank/dimensionality of the latent space of an autoencoder by deploying a nuclear norm regularizer that promotes a low-rank solution. We call this model \textit{Low-Rank Autoencoder} (LoRAE). This method consists of inserting a single linear layer between the encoder and decoder of a vanilla autoencoder. This layer is trained along the encoder and decoder networks. The primary objective of this methodology involves not only minimizing the conventional reconstruction loss but also minimizing the nuclear norm of the added linear layer. Due to the presence of nuclear norm minimization in the loss function, the network will now adjust to an effective low-dimensionality of latent space. The nuclear norm regularization of a matrix $\textbf{A}$ is an $l_{1}$ regularization of the singular values of $\textbf{A}$, and it, therefore, promotes a low-rank solution~\cite{recht2010guaranteed}. Similar to various regularization techniques, the additional linear layer remains inactive during inference. Consequently, the architecture of both the encoder and the decoder within the model remains consistent with the original design. In practical application, the linear layer (matrix) is treated as the last layer of the encoder during the inference phase.

We showcase the superiority of LoRAE in learning representations, surpassing the performance of conventional AE, VAE, IRMAE and several state-of-the-art deterministic autoencoders. This validation is carried out using the MNIST and CelebA datasets, employing diverse tasks such as generating samples from noise, interpolation, and classification. We additionally perform experiments to probe the influence of the regularizer on the rank of the latent space\footnote{The rank of the latent space corresponds to the count of non-zero singular values of the empirical covariance matrix of the latent space.}, achieved by varying its penalty parameter. Furthermore, we explore the impact of varying the dimension of the latent space while maintaining a constant penalty term.

The fundamental objective of this paper is to highlight the capacity of low-rank latent spaces acquired through the encoder to yield substantial data representations, thereby elevating both the generative potential and downstream performance. Moreover, we assert that the presence of a nuclear norm penalty in LoRAE offers several robust mathematical assurances. \emph{It's important to note that our paper does not strive to propose a new generative model. Only one of the essences of this work lies in empirically confirming that a low-rank constraint on the latent space of a simple deterministic autoencoder promotes generative capabilities comparable to well-established generative models.}

This paper serves as a proclamation that the utilization of low-rank autoencoders with nuclear norm penalty can yield significant representation outcomes while simultaneously affording the opportunity for the establishment of robust mathematical guarantees.

We summarize our contribution as follows:
\begin{enumerate}
    \item We introduce a novel framework to enrich the latent representation of autoencoders. This is achieved by incorporating a sparse/low-rank regularized projection layer, which dynamically reduces the dimensionality of the latent space.
    \item We provide substantial mathematical underpinning for LoRAE through an analysis of distance prediction error bounds. This analysis sheds light on the reasons driving LoRAE's superior performance in downstream tasks such as classification. Additionally, we offer theoretical guarantees (under certain assumptions) on the convergence of our learning algorithm.
    \item We showcased LoRAE's superior performance by comparing it with several baselines, like \textbf{(i)} a conventional deterministic AE, \textbf{(ii)} a VAE, and \textbf{(iii)} an IRMAE. We also conducted comparisons of LoRAE against several state-of-the-art deterministic autoencoders and various established \emph{generative adversarial network} (GAN) and VAE based generative models
    across a variety of generative tasks.
    
    
\end{enumerate}

\section{Literature Survey}


Training a fully linear network with SGD naturally results in a low-rank solution. This phenomenon can be interpreted as a form of implicit regularization, which has been extensively investigated across diverse learning tasks. Examples include deep matrix factorization~\cite{arora2019implicit, gunasekar2017implicit}, convolutional neural networks~\cite{gunasekar2018implicit}, and logistic regression~\cite{soudry2018implicit}. The implicit regularization offered by gradient descent is believed to be a pivotal element in enhancing the generalizability of deep neural networks. In the domain of deep matrix factorization, 
\emph{Arora et al.}~\cite{arora2019implicit} extended this concept in the case of deep neural nets with solid theoretical and empirical results that a deep linear network can promote low-rank solutions. \emph{Gunasekar et al.}~\cite{gunasekar2018implicit} further extended this work towards \emph{convolutional neural networks} (CNN) and proved that deep linear CNNs can derive low-rank solutions when optimized with gradient descent. Numerous other studies have concentrated on linear scenarios, aiming to empirically and theoretically investigate this phenomenon. The work by \emph{Saxe et al.}~\cite{saxe2019mathematical} demonstrates theoretically that a simple two-layer linear regression model can attain a low-rank solution when optimized using continuous gradient descent. Later \emph{Gidel et al.}~\cite{gidel2019implicit} extended this same concept of low-rank solutions in linear regression problems for a discrete case of gradient descent.
\begin{figure*}[!ht]
    \centering
    \includegraphics[width=0.95\textwidth ]{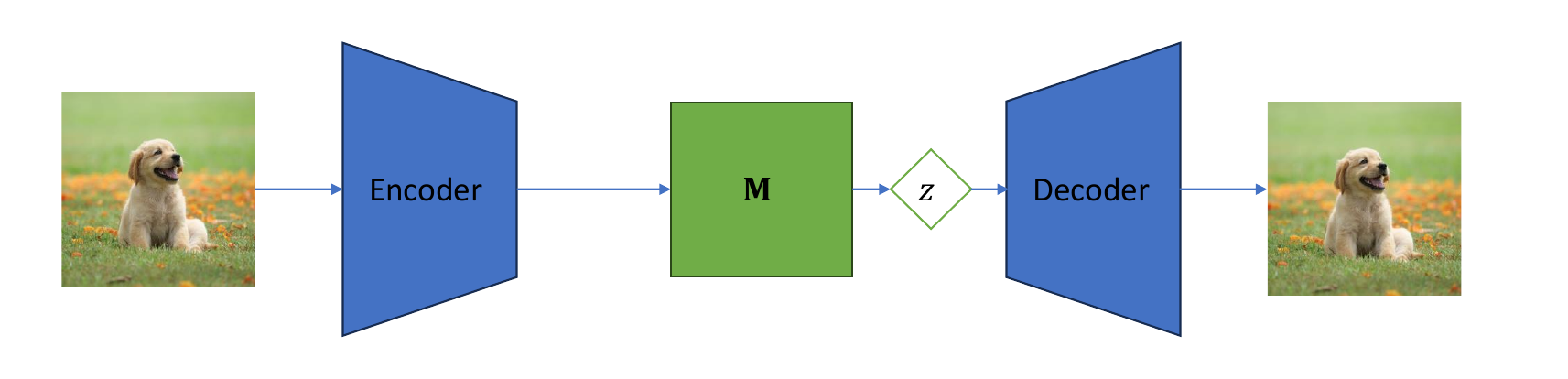}
    \caption{\textbf{Low-Rank Autoencoder (LoRAE)}: We employ a simple deterministic autoencoder coupled with a nuclear norm penalty in its loss function to facilitate the acquisition of a low-rank latent space. During the inference phase, the linear matrix \textbf{M} is condensed into the final layer of the encoder.}
    \label{fig:ERMAE}
\end{figure*}

While the previously mentioned works primarily aimed to comprehend how gradient descent contributes to generalization within established practices, our approach diverges by harnessing this phenomenon to design deep models capable of learning better latent representations of data. Autoencoders are simple yet powerful unsupervised deep models capable of learning latent representations of data. As most complex data (like images) are embedded in some low-dimensional subspace, it is necessary to limit their latent capacity. A significant category of these methods is founded on the concepts of variational autoencoders~\cite{kingma2013auto}, including variations like $\beta$-VAE~\cite{higgins2016beta}. A notable limitation of these methods arises from their probabilistic/generative nature, often resulting in the generation of reconstructed images with reduced clarity (blurry images). Conversely, basic deterministic autoencoders encounter an issue characterized by the existence of \emph{"holes"}\footnote{Discontinuous regions in the latent space of an autoencoder.} within their latent space, primarily arising from the absence of constraints on the distribution of its latent space. Several deterministic autoencoders are proposed to tackle this issue, namely \emph{regularized autoencoders} (RAE)~\cite{ghosh2019variational}, \emph{wasserstein autoecnoder} (WAE)~\cite{tolstikhin2017wasserstein} and \emph{vector quantized VAE} (VQ-VAE)~\cite{van2017neural}. Recently, \emph{Jing et al.}~\cite{jing2020implicit} proposed \emph{implicit rank minimizing autoencoder} (IRMAE) which leverages the \emph{"low-rank"} phenomena~\cite{arora2019implicit} in deep linear networks to learn a better latent representation. It uses a series of linear layers sandwiched between the encoder and decoder of a simple autoencoder. A sequence of linear layers maintains functional and expressive parity with a single linear layer. However, it's essential to note that implicit regularization doesn't manifest for unstructured datasets like random full-rank noise, as pointed out in~\cite{arora2019implicit}. This observation suggests that the occurrence of this phenomenon hinges on the underlying data structure. Moreover, it lacks vital mathematical assurances such as the convergence of its iterations and the underlying principles driving its effectiveness in downstream tasks.
\section{Low-Rank Autoencoder (LoRAE)}
In this section, we unveil our proposed architecture. Let $\textbf{E}: \mathbb{R}^{m \times n \times c} \rightarrow \mathbb{R}^{l}$ and $\textbf{D}: \mathbb{R}^{l} \rightarrow \mathbb{R}^{m \times n \times c}$ denote the encoder and decoder of a simple deterministic autoencoder respectively. Let $z \in \mathbb{R}^l$ denote a vector in its latent space, where $l$ is the dimension of latent space. The latent space is modelled by $\textbf{E}(x)$. Here $x \in \mathbb{R}^{m \times n \times c}$ be an image of size $m \times n$ with $c$ number of channels. A simple (vanilla) deterministic autoencoder optimizes the $L_{2}$ reconstruction loss $\mathcal{L}_{vanilla} = \|\ x - \textbf{D}(\textbf{E}(x)) \|\ _{2}^{2}$ without any constrain over its latent space, hence promoting the presence of holes. 

In LoRAE, we add an additional single linear layer between the encoder and decoder. Let $\textbf{M} \in \mathbb{R}^{l \times l}$ denote a real matrix characterized by a linear layer. The diagram of LoRAE is shown in Figure~\ref{fig:ERMAE}. We explicitly regularize the matrix $\textbf{M}$ with a nuclear norm penalty to encourage learning a low-rank latent space. We train the matrix $\textbf{M}$ jointly with the encoder and decoder. Hence, the final loss of LoRAE can be written as:
\begin{equation}
\label{eq:loss}
    \mathcal{L}(\textbf{E},\textbf{D},\textbf{M}) = \|\ x - \textbf{D}(\textbf{M}(\textbf{E}(x))) \|\ _{2}^{2} + \lambda \|\ \textbf{M} \|\ _{*} 
\end{equation}
where $\|\ \textbf{M} \|\ _{*}$ is the nuclear norm\footnote{$\|\ \textbf{M} \|\ _{*} = \sum_{i=1}^{l}\sigma_{i}(\textbf{M}) = trace(\sqrt{\textbf{M}^{T}\textbf{M}})$} (also known as the trace norm) of matrix $\textbf{M}$. We will now minimize the loss given in Eq.(\ref{eq:loss}) using ADAM~\cite{kingma2014adam} optimizer in batch form. Let $\mathcal{B}$ denote a mini-batch of training data and $|\mathcal{B}|$ be the number of training points in it that are randomly sampled from the training set. We can now write Eq.(\ref{eq:loss}) in mini-batch form as follows:
\begin{equation}
\label{eq:batch_loss}
    \mathcal{L}_{\mathcal{B}}(\textbf{E},\textbf{D},\textbf{M}) = \frac{1}{|\mathcal{B}|}\sum_{x \in \mathcal{B}} \|\ x - \textbf{D}(\textbf{M}(\textbf{E}(x))) \|\ _{2}^{2} + \lambda \|\ \textbf{M} \|\ _{*}
\end{equation}

Let $\theta \in \{\textbf{E}, \textbf{D}, \textbf{M}\}$, $\beta_{1}$ and $\beta_{2}$ $\in [0,1)$ and $\epsilon > 0$. Here, $\textbf{m}_{\theta, t} = \beta_{1}\textbf{m}_{\theta, t-1} + (1 - \beta_{1})\nabla_{\theta}\mathcal{L}_{batch}$, $\textbf{v}_{\theta, t} = \beta_{2}\textbf{v}_{\theta, t-1} + (1 - \beta_{2})(\nabla_{\theta}\mathcal{L}_{batch})^{2}$
and $V_{\theta, t} = diag(\textbf{v}_{\theta, t})$. Now, we can solve for Eq.(\ref{eq:batch_loss}) using Algorithm~\ref{alg:training}.
\begin{algorithm}
\caption{Minimizing Eq. (2) using ADAM}
\label{alg:training}
\KwIn{Training data $\{x\}_{i=1}^{N}$, Batch size $|\mathcal{B}|$ $\in$ $\mathbb{Z}^{+}$, Learning rate: $\alpha \in (0,1]$}
\textbf{Initialization:}\\
\For{$t$ \textbf{from} 1 \textbf{to} $T$:}{
    (i) Randomly sample $|\mathcal{B}|$ number of data points from training set.\\
    (ii) Compute gradient of $\mathcal{L}_{batch}$ with respect to \textbf{E}, \textbf{D} and \textbf{M}.\\
    (iii) Update the parameters using the ADAM update rule:\\
    $\textbf{E}_{t+1} = \textbf{E}_{t} - \alpha(V_{\textbf{E}, t}^{1/2} + diag(\epsilon\mathbb{I}))^{-1}\textbf{m}_{\textbf{E}, t}$ \\
    $\textbf{D}_{t+1} = \textbf{D}_{t} - \alpha(V_{\textbf{D}, t}^{1/2} + diag(\epsilon\mathbb{I}))^{-1}\textbf{m}_{\textbf{D}, t}$ \\
    $\textbf{M}_{t+1} = \textbf{M}_{t} - \alpha(V_{\textbf{M}, t}^{1/2} + diag(\epsilon\mathbb{I}))^{-1}\textbf{m}_{\textbf{M}, t}$ \\
    
}
\textbf{End:}\\
\textbf{Output:} The learned $\textbf{E}^{*}$, $\textbf{D}^{*}$ and $\textbf{M}^{*}$
\end{algorithm}

During the training phase, the use of the nuclear norm penalty prompts the latent variables to occupy an even lower-dimensional subspace. As a result, this process reduces the rank of the empirical covariance matrix of the latent space. To enhance the impact of this regularization, one can amplify its effect by adjusting the penalty term $\lambda$ to a higher value.

During inference time, the linear layer is \emph{"collapsed"} into the encoder. Hence, we can directly use this new encoder for generative tasks. We can also use the encoder only for downstream tasks.



\section{Experiments}
In this section, we conduct an empirical assessment of LoRAE in comparison to baselines (AE, VAE, IRMAE) and state-of-the-art models (WAE, RAE). In Section~\ref{sec:dim_red}, we highlight that LoRAE occupies a relatively smaller latent space compared to a standard AE and achieves smooth dimensionality reduction, unlike IRMAE. Moving to Section~\ref{sec:gen}, we provide empirical evidence of LoRAE's ability to generate images of superior quality compared to the fundamental vanilla AE. This superiority is attributed to its low-rank latent space. Additionally, the model demonstrates enhanced quantitative performance when compared against simple AE, VAE, and IRMAE. Next, in Section~\ref{sec:classification}, we leverage the encoder component of our trained model to effectively classify images within the MNIST dataset. Lastly, we study the effect of two crucial hyper-parameters of our model in Section~\ref{sec:addl_exp}.

\subsection{Dimensionality Reduction}\label{sec:dim_red}
In Figure~\ref{fig:both_figures}, we present the dimensionality reduction achieved by LoRAE and IRMAE in the latent space. LoRAE's gradual and smooth decay in the singular value plot stands in contrast to IRMAE's sharp decline. This discrepancy contributes to IRMAE occupying an even smaller-dimensional latent space than LoRAE. This distinction arises because nuclear norm minimization penalizes larger singular values more. Conversely, in the IRMAE plot, a rapid transition toward zero is observed compared to ours, potentially stemming from the fact that gradient descent dynamics penalize smaller singular values more than larger ones, driving them toward zero. LoRAE employs a considerably smaller latent space compared to a basic deterministic AE.

In Section~\ref{sec:effect_of_lambda}, we showcased that there exists a \emph{sweet} point of latent space rank which signifies a balance wherein the model's generative capacity is empirically maximized. Hence, latent spaces with very low ranks/low dimensions are also unfavourable as they lead to a deterioration in generative performance.
\subsection{Generative Tasks}
\label{sec:gen}
The effectiveness and quality of a learned latent space can be assessed by generating images from it and examining the smoothness during transitions from one point to another within it. When the generated images demonstrate a high level of quality, it indicates the latent space's successful learning. We trained our model using the MNIST and CelebA datasets and conducted a comparison with a basic deterministic AE, IRMAE as well as a VAE. We set the latent space dimensions to 128 for MNIST and CelebA, training for 50 and 100 epochs respectively (Additional hyperparameter details are available in our \emph{supplementary material}). In this section, we conducted two experiments: \textbf{(i)} linearly interpolating between two data points, and \textbf{(ii)} generating images from the latent space using \emph{Gaussian Mixture Model} (GMM) and \emph{Multivariate Gaussian} (MVG) noise fitting. Additionally, we quantitatively evaluated our model's generative capability using the \emph{Fréchet Inception Distance} (FID)~\cite{heusel2017gans, parmar2021cleanfid} score. Across the tasks mentioned above, we compare our model qualitatively (see Figure~\ref{fig:triangular_layout}, Table~\ref{tab:GMM_generation} and Table~\ref{lttab:MVG_generation}) and quantitatively (see Table~\ref{tab:99}). with a simple AE, VAE and IRMAE. All models reported in Table~\ref{tab:99} shared the exact same architecture.

\begin{figure}[] 
    \centering
    \begin{subfigure}{0.23\textwidth} 
        \centering
        \includegraphics[width=\linewidth]{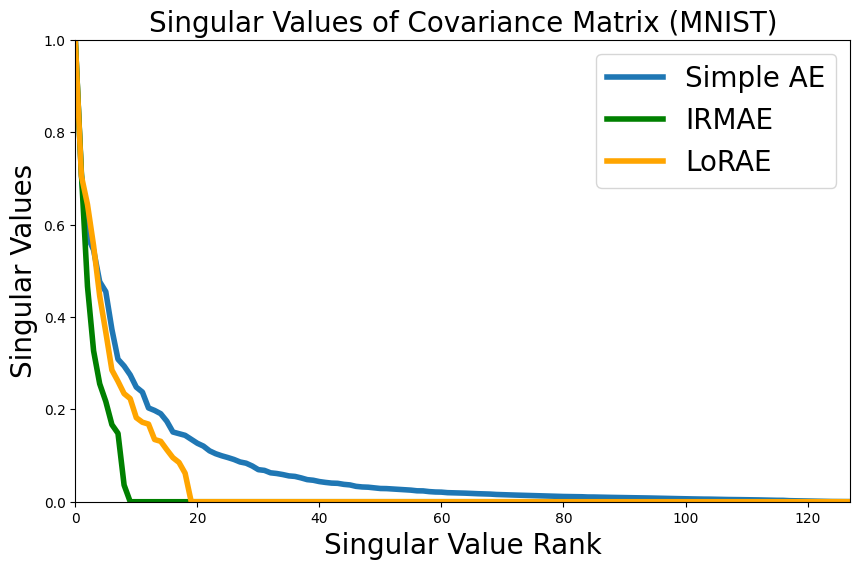}      
    \end{subfigure}%
    \begin{subfigure}{0.23\textwidth}
        \centering
        \includegraphics[width=\linewidth]{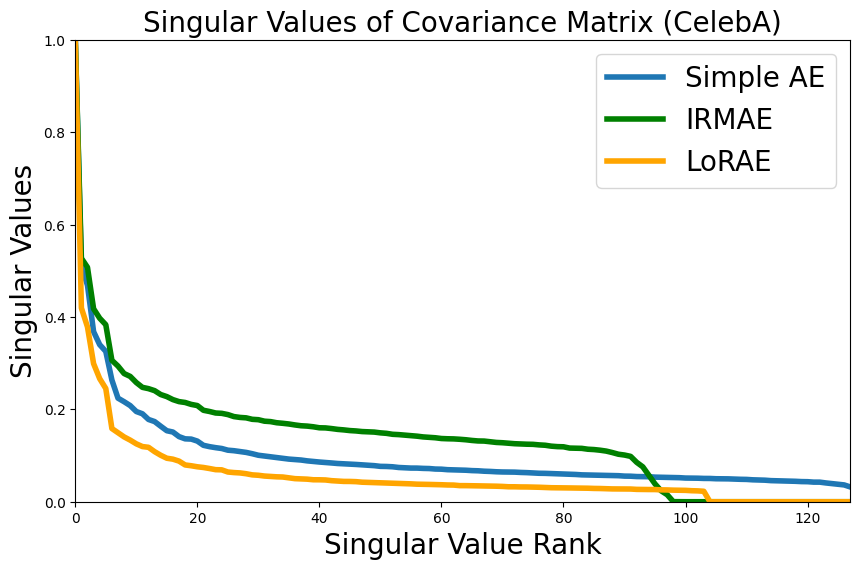}     
    \end{subfigure}  
\caption{\textbf{Singular value plot of the empirical covariance matrix of the latent space}: A comparison between a basic AE, IRMAE and LoRAE on MNIST and CelebA datasets. The empirical covariance matrix is computed from the test set of both datasets.}
\label{fig:both_figures}
\end{figure}

\begin{figure}[]
    \centering
    \begin{subfigure}{0.45\columnwidth}
        \includegraphics[width=\linewidth]{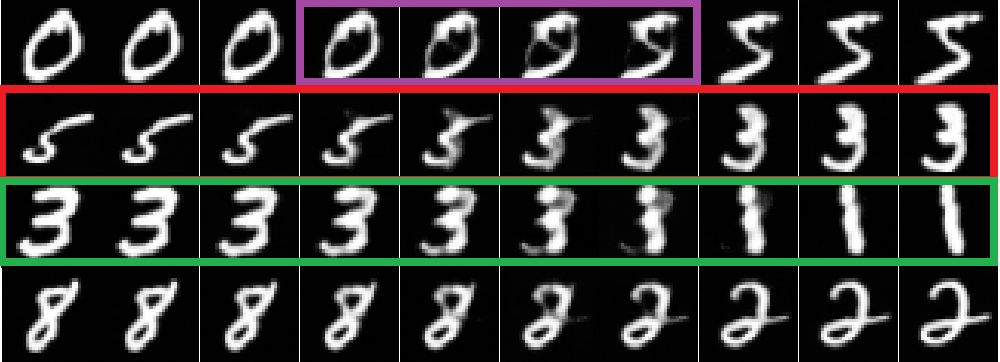}
        \caption{Simple AE}
    \end{subfigure}
    \hfill
    \begin{subfigure}{0.445\columnwidth}
        \includegraphics[width=\linewidth]{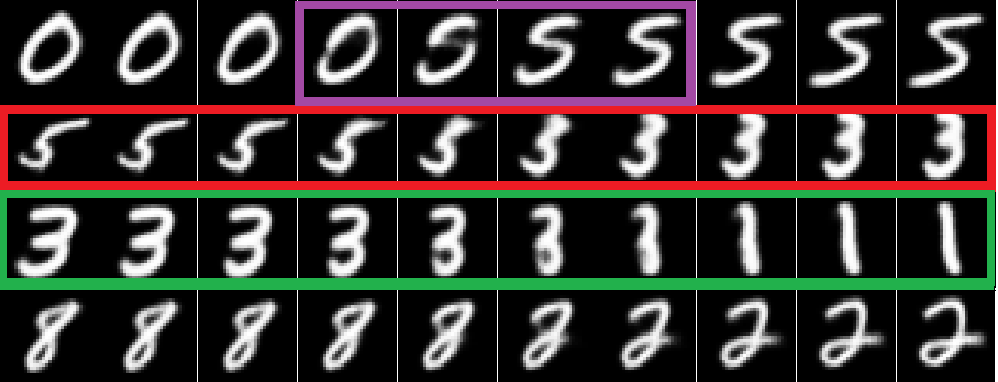}
        \caption{VAE}
    \end{subfigure}
    
    \begin{subfigure}{0.45\columnwidth}
        \includegraphics[width=\linewidth]{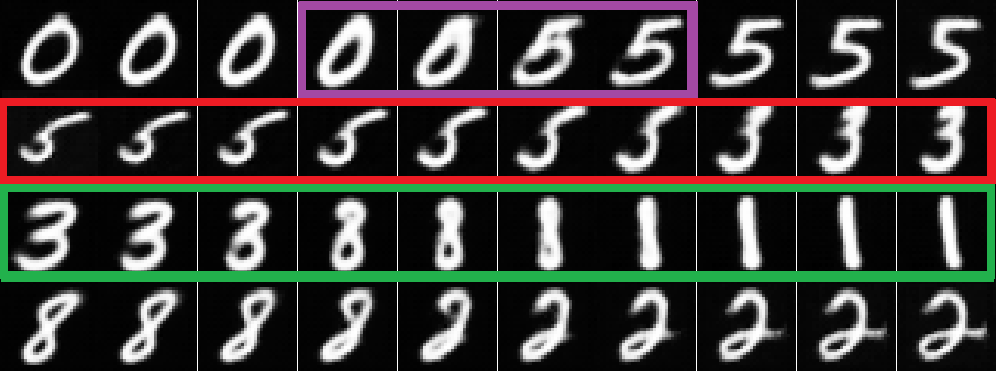}
        \caption{IRMAE}
    \end{subfigure}
    \hfill
    \begin{subfigure}{0.45\columnwidth}
        \includegraphics[width=\linewidth]{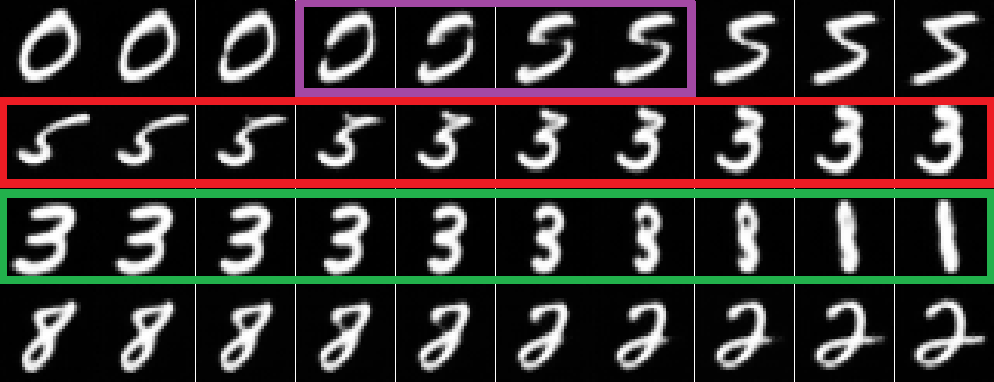}
        \caption{Ours}
    \end{subfigure}
    
    \caption{\textbf{Interpolation}: Employing linear interpolation among data points on the MNIST dataset, we observe three instances in a left-to-right top-to-bottom sequence: (a) Basic AE, (b) VAE, (c) IRMAE and (d) LoRAE.}
    \label{fig:triangular_layout}
\end{figure}

\subsubsection{Interpolation Between Datapoints}
\begin{table*}[htbp]
\centering
\caption{\textbf{Generation from GMM}: MNIST and CelebA images generated from GMM with 4 and 10 clusters respectively. We show the images generated by a simple autoencoder, VAE, and LoRAE (from left to right).}
\label{tab:GMM_generation}
\begin{tabular}{c c c c c}

& \textit{AE} & \textit{VAE} & \textit{IRMAE} & \textbf{Ours} \\

\textit{MNIST} &
\raisebox{-0.5\height}{\includegraphics[scale=0.27]{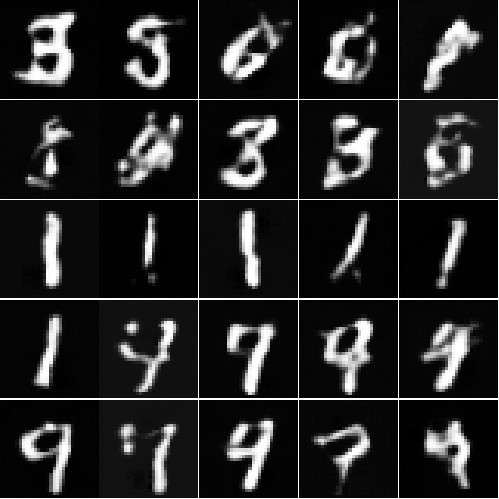}} &
\raisebox{-0.5\height}{\includegraphics[scale=0.27]{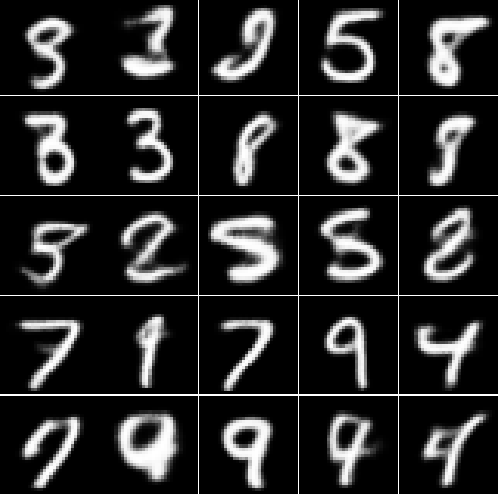}} &
\raisebox{-0.5\height}{\includegraphics[scale=0.27]{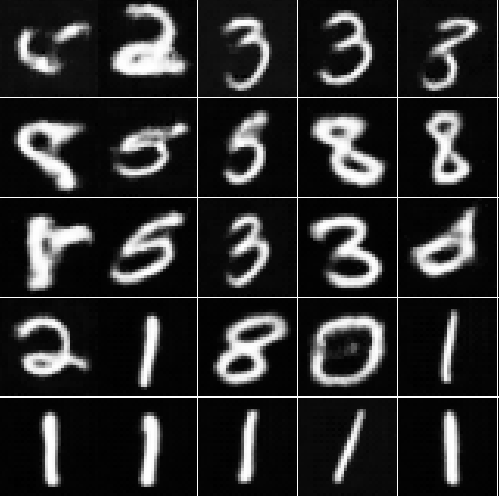}} &
\raisebox{-0.5\height}{\includegraphics[scale=0.27]{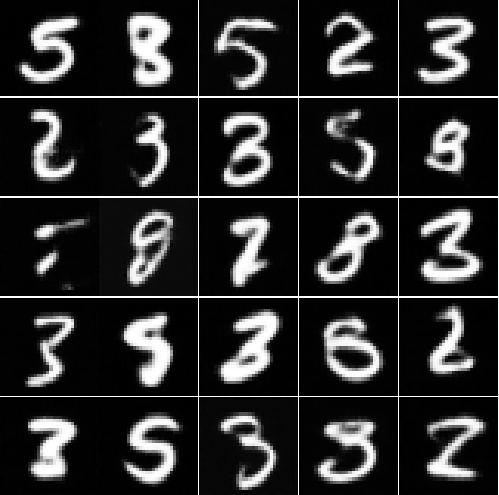}} \\
\addlinespace[0.65cm]

\textit{CelebA} &
\raisebox{-0.5\height}{\includegraphics[scale=0.27]{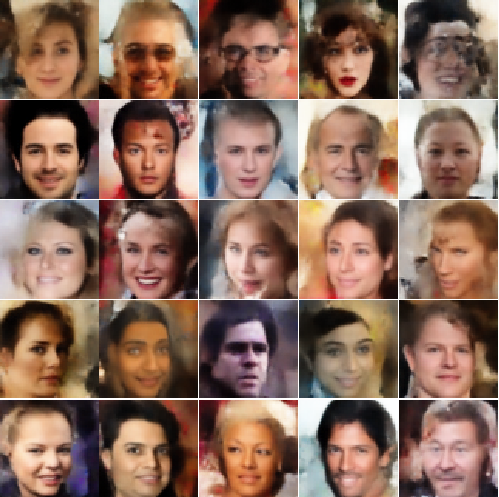}} &
\raisebox{-0.5\height}{\includegraphics[scale=0.27]{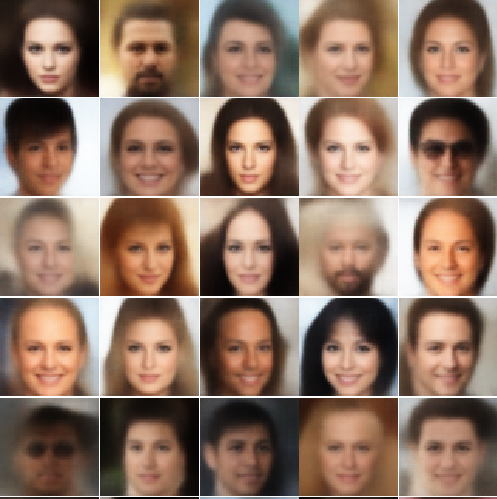}} &
\raisebox{-0.5\height}{\includegraphics[scale=0.27]{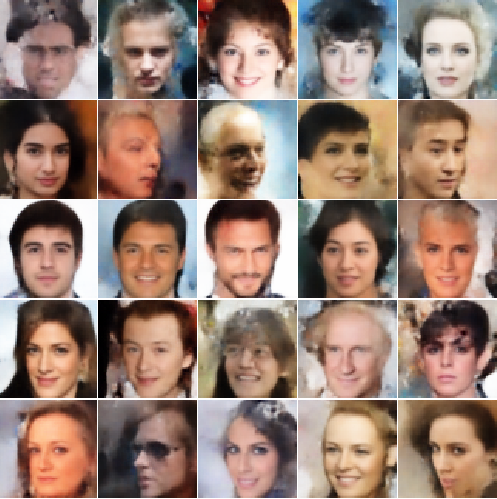}} &
\raisebox{-0.5\height}{\includegraphics[scale=0.27]{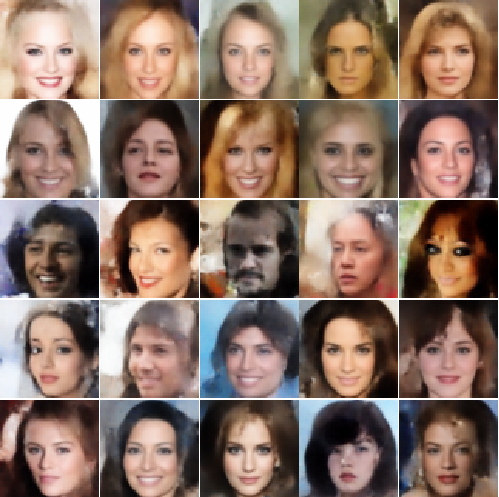}} \\
\end{tabular}
\end{table*}

\begin{table*}[htbp]
\centering
\caption{\textbf{Generation from MVG}: MNIST and CelebA images generated from MVG noise. We show the images generated by a simple autoencoder, VAE, and LoRAE (from left to right).}
\label{lttab:MVG_generation}
\begin{tabular}{c c c c c}

& \textit{AE} & \textit{VAE} & \textit{IRMAE} & \textbf{Ours} \\

\textit{MNIST} &
\raisebox{-0.5\height}{\includegraphics[scale=0.27]{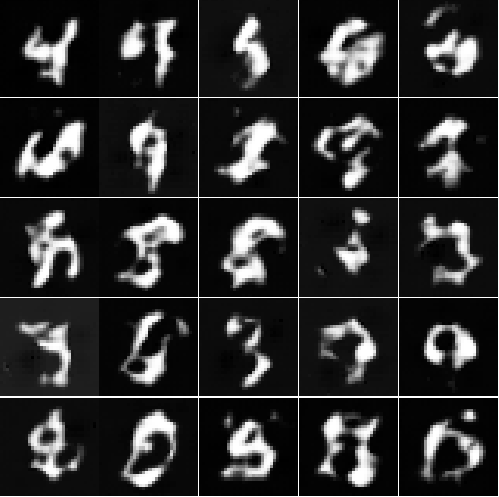}} &
\raisebox{-0.5\height}{\includegraphics[scale=0.27]{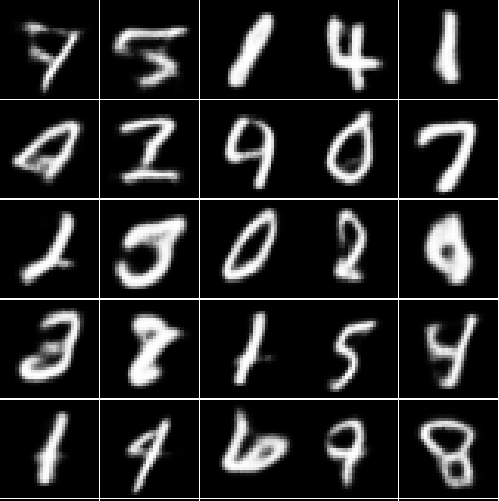}} &
\raisebox{-0.5\height}{\includegraphics[scale=0.27]{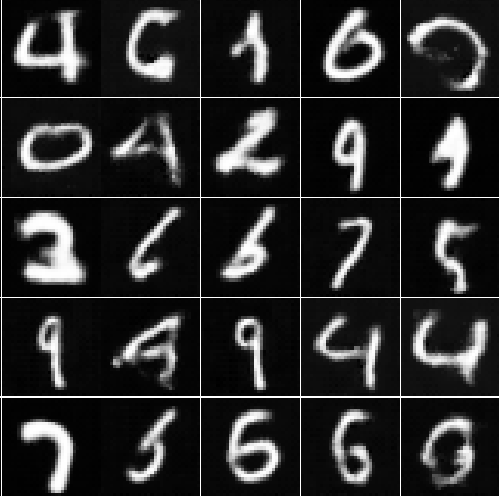}} &
\raisebox{-0.5\height}{\includegraphics[scale=0.25]{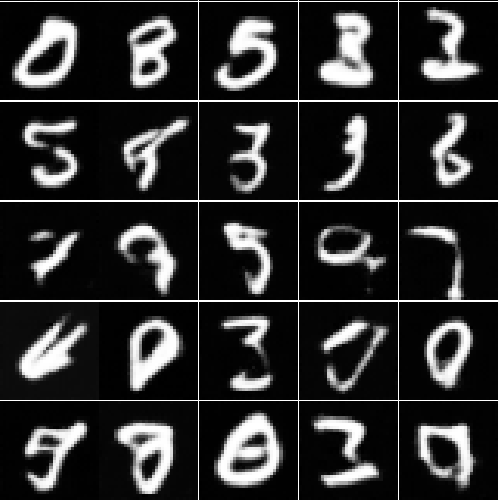}} \\
\addlinespace[0.65cm]

\textit{CelebA} &
\raisebox{-0.5\height}{\includegraphics[scale=0.27]{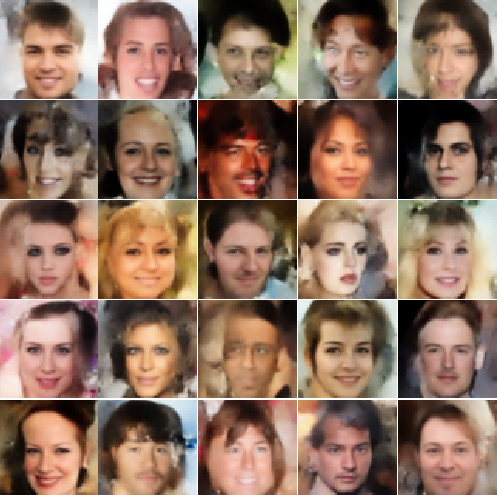}} &
\raisebox{-0.5\height}{\includegraphics[scale=0.27]{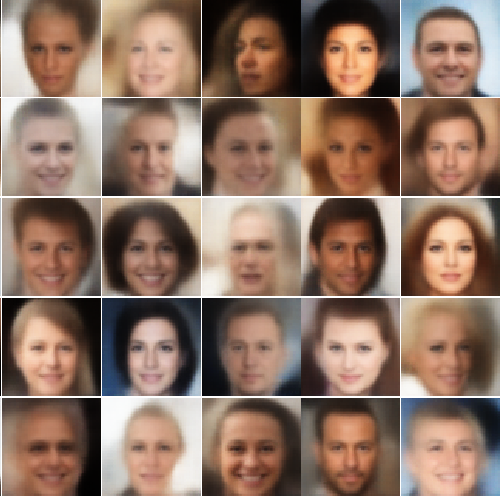}} &
\raisebox{-0.5\height}{\includegraphics[scale=0.27]{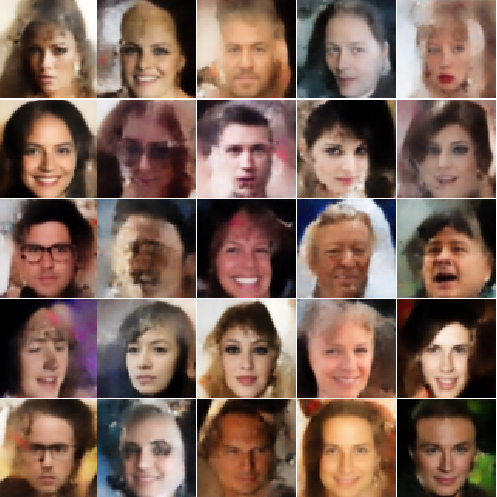}} &
\raisebox{-0.5\height}{\includegraphics[scale=0.27]{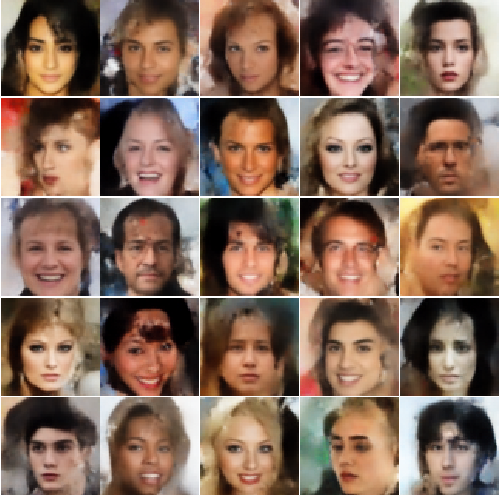}} \\
\end{tabular}
\end{table*}
We conduct linear interpolation on the latent variables of two randomly selected images from the test set. The results depicted in Figure~\ref{fig:triangular_layout} show that LoRAE significantly outperforms the conventional AE. Its generated quality aligns with that of a VAE and surpasses IRMAE. The progression from digit 5 to 3 showcases a high degree of smoothness in LoRAE, setting it apart from both VAE and IRMAE (highlighted by the red box). VAE's probabilistic nature introduces a slightly blurred transition for the same sequence. Conversely, transitioning from 3 to 1, IRMAE generates a digit resembling 8 (marked by the green box). Likewise, in the transition from 0 to 5, IRMAE produces a somewhat blurry transformation (indicated by the violet box in Figure~\ref{fig:triangular_layout} (c)). Remarkably, LoRAE achieves seamless transitions between individual digits, thereby effectively learning a better and refined latent space.
\subsubsection{Generating Images from Noise}
\label{sec:generation_images}

Deterministic autoencoders lack generative capability. Here, to demonstrate the generative capability of LoRAE, we fit a Gaussian noise. We demonstrate the capability of LoRAE to generate high-quality images from Gaussian noise. In particular, we accomplish this by fitting \textbf{(i)} a \emph{Gaussian Mixture Model} (GMM) and \textbf{(ii)} a \emph{Multivariate Gaussian (MVG)} (which we denote as $\mathcal{N}$) distribution to its latent space. After fitting the noise distribution, we proceed to sample from it and feed these samples through the decoder to generate images. We take the number of clusters 4 and 10 for MNIST and CelebA respectively in the case of GMM to avoid overfitting. We quantitatively evaluate the generating capacity of each model using the FID score.

Analyzing the outcomes presented in Table~\ref{tab:GMM_generation} and Table~\ref{lttab:MVG_generation}, a clear trend emerges: our model excels in generating images with higher visual quality when compared to the simple AE and IRMAE models. Furthermore, its image generation quality stands on par with that of the VAE. Differing from VAE, which often produces images with blurred backgrounds due to its probabilistic nature, our model, being deterministic, avoids generating images with such blurriness. This distinction results in clearer and more defined images produced by LoRAE. These findings are supported by the data showcased in Table~\ref{tab:99} where our model achieves the best FID score among the three, highlighting the pivotal role of low-rank latent space achieved through explicit regularization using nuclear norm penalty. This strategic enhancement of generative capabilities effectively surpasses even the performance of a conventional VAE.

\begin{table}[h]
    \centering
    \caption{FID scores obtained between the generated samples from MVG and GMM noise and real test samples for all models on MNIST and CelebA datasets.}
    \label{tab:99}
    \begin{tabular}{c c c c c}
        \hline
         & \multicolumn{2}{c}{\textbf{MNIST}} & \multicolumn{2}{c}{\textbf{CelebA}} \\
        \cline{2-5}
         & $\mathcal{N}$ & GMM & $\mathcal{N}$ & GMM \\
        \hline
        \hline
        AE & 103.08 & 68.97 & 68.13 & 59.43 \\
        \hline
        VAE & 21.01 & 18.86 & 61.87 & 53.63 \\
        \hline
        IRMAE & 26.58 & 22.31 & 58.98 & 48.56\\
        \hline
        \textbf{LoRAE} & \textbf{19.50} & \textbf{11.09} & \textbf{56.29} & \textbf{45.43}\\
        \hline
    \end{tabular}
\end{table}

\vspace{-14pt}
\subsubsection{Comparing with other SOTA Models}
We perform a comparative analysis of LoRAE against state-of-the-art deterministic autoencoders such as \emph{wasserstein's autoencoder} (WAE)~\cite{tolstikhin2017wasserstein} and \emph{regularized autoencoders}
(RAE)~\cite{ghosh2019variational}, in terms of FID scores (we retrained these models in our system using the exact parameters from their original paper.). Based on the data presented in Table~\ref{tab:4}, it can be observed that LoRAE attains the top rank in all cases, except for the MVG ($\mathcal{N}$) case of CelebA, where it shares the second position with RAE. 

It is evident that LoRAE exhibits enhanced performance on the MNIST dataset and remains competitively aligned with the performance on the CelebA dataset.

\begin{table}[!htbp]
    \centering
    \caption{Comparision of LoRAE against modern state-of-the-art deterministic autoencoders (FID Score is used for comparison).}
    \label{tab:4}
    \begin{tabular}{c c c c c}
        \hline
         & \multicolumn{2}{c}{\textbf{MNIST}} & \multicolumn{2}{c}{\textbf{CelebA}} \\
        \cline{2-5}
         & $\mathcal{N}$ & GMM & $\mathcal{N}$ & GMM \\
        \hline
        \hline
        WAE & 21.04 & 11.32 & 57.6 & 45.91 \\
        \hline
        RAE & 22.12 & 11.54 & \textbf{50.31} & 46.05 \\
        \hline
        \textbf{LoRAE} & \textbf{19.50} & \textbf{11.09} & 56.29 & \textbf{45.43}\\
        \hline
    \end{tabular}
\end{table}

Moreover, we broaden our comparative analysis to include various other generative models on the CelebA dataset. This encompasses fundamental models like the \emph{generative adversarial network} (GAN)~\cite{goodfellow2014generative}, \emph{least squares GAN} (LS-GAN)~\cite{mao2017least}, \emph{non-saturating GAN} (NS-GAN)~]\cite{fedus2017many}, along with the combined architecture of VAE and Flow models (VAE+Flow)~\cite{rezende2015variational}. 

The reported values for all GAN-based models in this Table~\ref{tab:5} correspond to the optimal FID results from an extensive hyperparameter search, conducted separately for each dataset, as detailed by \emph{Lucic et al.} in ~\cite{lucic2018gans}. Instances involving severe mode collapse were excluded to prevent inflation of these FID scores.
\begin{table}[!htbp]
\centering
\caption{Comparison of LoRAE against well-known generative models on CelebA dataset. (FID Score is used for comparison)}
\label{tab:5}
\begin{tabular}{c c c c c}
\hline
GAN & LS-GAN & NS-GAN & VAE + Flow & \textbf{LoRAE} \\
\hline
\hline
65.2 & \textbf{54.1} & 57.3 & 65.7 & 56.3 \\
\hline
\end{tabular}
\end{table}
Our analysis clearly demonstrates that LoRAE achieves FID scores that align closely with those attained by other established generative models. 

To conclude Section~\ref{sec:gen}, we offer empirical evidence that the incorporation of a nuclear norm penalty term into a simple AE framework, which enforces an explicit low-rank constraint on the latent space, leads to a substantial enhancement in the model's generative capabilities.
\subsection{Downstream Classification Task}
\label{sec:classification}
Latent variables play a crucial role in downstream tasks as they encapsulate the fundamental underlying structure of the data distribution~\cite{he2020momentum, misra2020self, chen2020simple}. The potential of these self-supervised learning methods to outperform purely-supervised models is particularly promising. We engage in downstream classification using the MNIST dataset. For each method - the simple AE, VAE, and LoRAE - we train a \emph{multi-layer perceptron} (MLP) layer on top of the trained encoder (Parameters for classification are given in our \emph{supplementary material}). The fine-tuning of the pre-trained encoder is excluded, except for instance involving the purely supervised version of the basic AE (tagged as \emph{supervised} in Table~\ref{tab:classification_table}). As shown in Table~\ref{tab:classification_table}, representation learned by LoRAE showcases significantly superior classification accuracy in comparison to simple AE and VAE. Hence, it is evident that LoRAE surpasses simple AE, VAE and the supervised version in the low-data regime while maintaining a similar level of performance when operating on full-length datasets. 
\begin{table}[htbp]
\centering
\caption{Classification accuracy (in \%) obtained from a simple AE, VAE and LoRAE on MNIST dataset.}
\label{tab:classification_table}
\begin{tabular}{c c c c c c c}
\hline
Size of\\training set & 10 & 100 & 1000 & 10,000 & 60,000 \\
\hline
\hline
AE &
41.0 & 68.16 & 89.8 & 96.5 & 98.1 \\

\hline
VAE &
41.5 & 77.47 & 94.0 & \textbf{98.5} & 98.9 \\
\hline

\textbf{LoRAE} &
\textbf{46.6} & \textbf{89.02} & \textbf{95.4} & 97.9 & 98.6\\
\hline
\hline
Supervised &
37.8 & 73.59 & 94.2 & 98.3 & \textbf{99.2}\\
\hline
\end{tabular}
\end{table}

We provide strong theoretical support for its improved performance in downstream classification tasks through Theorem 2 in Section~\ref{sec:downstream}.

\subsection{Effect of Hyperparameters in LoRAE}
\label{sec:addl_exp}
In this section, we investigate two crucial hyperparameters: \textbf{(i)} penalty parameter $\lambda$ and \textbf{(ii)} encoder output dimension (latent dimension). We seek insights into their impact on LoRAE's generative capacity.

\subsubsection{Effect of Penalty Parameter $\lambda$ on FID}
\label{sec:effect_of_lambda}
Elevating the penalty term results in a more pronounced regularization effect, whereas reducing it lessens this impact. We investigate the effect of the regularization penalty parameter on FID scores. Intuitively, the rank of the latent space appears to have an inverse relationship with the penalty parameter. A more pronounced emphasis on nuclear norm minimization is likely to lead to a reduction in the rank of the latent space (more discussion in Section~\ref{sec:rank_latent_space}). 

 
\begin{table}[!htbp]
\centering
\caption{Effect of varying penalty parameter for MNIST dataset.}
\label{tab:lambda_vary}
\begin{tabular}{c c c c c c}
\hline
$\lambda$ & $10^{-1}$ & $10^{-2}$ & $10^{-3}$ & $10^{-4}$ & $10^{-5}$\\
\hline
\hline
Rank of\\latent space & \multicolumn{1}{c}{$7$} & \multicolumn{1}{c}{$8$} & \multicolumn{1}{c}{$\textbf{18}$} & \multicolumn{1}{c}{$40$} & \multicolumn{1}{c}{$65$} \\
\hline
FID & 23.65 & 22.18 & \textbf{11.09} & 25.60 & 40.75 \\
\hline
\end{tabular}
\end{table}

In Table~\ref{tab:lambda_vary}, it becomes evident that extremely low and high-dimensional latent spaces result in suboptimal generative capabilities for LoRAE. This highlights the significance of identifying an optimal penalty parameter ($\lambda$) that yields the most suitable latent space for generative tasks, as indicated by the best FID. Hence, the hyperparameter $(\lambda)$ assumes a critical role and requires optimization in practical applications.


\subsubsection{Comparison with Simple AE on Various Latent Dimension}
Autoencoders that exhibit diverse latent dimensions or varying prior configurations inherently necessitate a balance between obtaining valuable representations. In this context, we explore the impact of latent dimensionality on FID scores for both LoRAE and a simple AE. 
\begin{table}[!htbp]
\centering
\caption{Effect of varying latent dimensions on LoRAE compared to a simple AE using CelebA. (FID Score is used for comparison)}
\label{tab:output_dim_variation}
\begin{tabular}{c c c c c c}
\hline
Latent \\Dimension & 64 & 128 & 256 & 512 & 1024 \\
\hline
\hline
AE & \textbf{69.08} & 68.13 & 66.94 & 91.42 & 107.98 \\
\hline
\textbf{LoRAE} & 71.58 & \textbf{56.29} & \textbf{62.42} & \textbf{62.93} & \textbf{63.89} \\
\hline
\end{tabular}
\end{table}
From Table~\ref{tab:output_dim_variation}, it becomes apparent that LoRAE, when endowed with larger latent dimensions, demonstrates enhanced performance compared to the optimally dimensioned AE. Additionally, it is noticeable that as the dimensionality increases, the performance of LoRAE tends to plateau.
\section{Theoretical Analysis}
In this section, we offer an exhaustive theoretical examination of our proposed model. Our focus revolves around two key aspects: firstly, an exploration of the rate of convergence exhibited by our learning algorithm, and secondly, an exploration of the lower bound on the min-max distance ratio. Furthermore, we also establish a proof demonstrating the inverse correlation between the rank of LoRAE's latent space and the penalty parameter ($\lambda$). Collectively, these analyses highlight and confirm the effectiveness of LoRAE.. All proofs are given in our \textit{supplementary materials}.
\subsection{Convergence Analysis}
In this section, we are going to present the convergence analysis for the loss function as described in Eq.(\ref{eq:loss}), considering the ADAM iterations outlined in Algorithm~\ref{alg:training}. Before delving into the analysis, we will outline the assumptions that have been considered. \textbf{(i)} The loss function in Eq.(\ref{eq:loss}) is $K - Lipchitz$, \textbf{(ii)} it has a $\sigma > 0$ bounded gradient, i.e $\|\nabla\mathcal{L}(\textbf{E}, \textbf{D}, \textbf{M})\|_{2} < \sigma < \infty$, \textbf{(iii)} it has a well-defined minima, i.e $\mathcal{L}(\textbf{E}^{*}, \textbf{D}^{*}, \textbf{M}^{*}) \leq \mathcal{L}(\textbf{E}, \textbf{D}, \textbf{M})$, where $\textbf{E}^{*}, \textbf{D}^{*}, \textbf{M}^{*} =
\argmin_{\textbf{E}, \textbf{D}, \textbf{M}} \mathcal{L}(\textbf{E}, \textbf{D}, \textbf{M})$ and \textbf{(iv)} we prove convergence for deterministic version of ADAM.

\begin{theorem}
Let the loss function $\mathcal{L}(\textbf{E}, \textbf{D}, \textbf{M})$ be $K-$Lipchitz and let $\gamma < \infty$ be an upper bound on the norm of the gradient of $\mathcal{L}$. Then the following holds for the deterministic version (when batch size = total dataset) of Algorithm (1):

For any $\sigma > 0$ if we let $\alpha = \sqrt{2(\mathcal{L}(\textbf{E}_{0}, \textbf{D}_{0}, \textbf{M}_{0}) - \mathcal{L}(\textbf{E}^{*}, \textbf{D}^{*}, \textbf{M}^{*}))/K\delta^{2}T}$, then there exists a natural number $T(\sigma,\delta)$ (depends on $\sigma$ and $\delta$) such that $\|\mathcal{L}(\textbf{E}_{t}, \textbf{D}_{t}, \textbf{M}_{t})\|_{2} \leq \sigma$ for some $t \geq T(\sigma,\delta)$, where $\delta^{2} = \frac{\gamma^{2}}{\epsilon^{2}}$.
\end{theorem}
With our analysis, we showed that our algorithm attains convergence to a stationary point with rate $\mathcal{O}(\frac{1}{T^{1/4}})$ when proper learning rate $\alpha > 0$ is set.

Therefore, we are in a position to confidently assert that the integration of the supplementary linear layer \textbf{M} and the utilization of nuclear norm regularization indeed ensure the guaranteed convergence of our learning algorithm.
\subsection{Lower Bound on min-max Distance Ratio}
\label{sec:downstream}
Considering a given set of independent and identically distributed (i.i.d) random variables, denoted as $x, x_{1}, x_{2}, \dots, x_{N}$, where each of these variables lies in the Euclidean space $\mathbb{R}^{d}$ (with $x$ and $x_i$ applicable for all $i \in \{1, 2, \dots, N\}$), the corresponding embeddings $f^{*}(x)$ and $f^{*}(x_{1 \leq i \leq N})$ also demonstrate an i.i.d nature, irrespective of the method used for their learning as highlighted by \emph{Dasgupta et al.} in \cite{dasgupta2020improving}. \emph{Beyer et al.} in \cite{beyer1999nearest} postulated that, in high-dimensional spaces, the minimum distance $\left(d_{min}^{f^{*}} = \underset{1 \leq j \leq N}{\min} d^{f^{*}}(x,x_{j})\right)$\footnote{Here, $d^{f^{*}}(.,.)$ is some distance metric/measure.} and the maximum distance $\left(d_{max}^{f^{*}} = \underset{1 \leq j \leq N}{\max} d^{f^{*}}(x,x_{j})\right)$ tend to be similar when considering $f(x)$ and $f(x_{1 \leq i \leq N})$ as i.i.d. Hence, the notion of similarity and dissimilarity with respect to the distance function between data points in embedding space is completely lost in such cases. In autoencoders, to avoid learning high-dimensional latent features, one may directly reduce the dimensionality of the output layer of the encoder, but this will cause dimensional collapse. In our model, we consider the matrix $\textbf{M}$ to transform the encoder embedding result $\textbf{E}(x)$ into the latent vector $\textbf{M}(\textbf{E}(x))$. When we further introduce the low-rank constraint for \textbf{M}, we can obtain a low-dimensional latent space for a simple autoencoder.

Given that our approach explicitly imposes a constraint on the dimensionality of the learned latent space, it follows intuitively that the min-max distance ratio, i.e., $(d_{max}^{\textbf{E}^{*}\textbf{M}^{*}} - d_{min}^{\textbf{E}^{*}\textbf{M}^{*}})/d_{min}^{\textbf{E}^{*}\textbf{M}^{*}}$, should invariably possess a lower-bound.

\begin{theorem}
    Given any set of i.i.d x, $x_{1}, x_{2}, \dots, x_{N}$ $\in \mathbb{R}^{l}$, we denote $d_{max}^{\textbf{E}^{*}\textbf{M}^{*}} = \underset{1 \leq j \leq N}{\max} d^{\textbf{E}^{*}\textbf{M}^{*}}(x,x_{j})$ and\\ $d_{min}^{\textbf{E}^{*}\textbf{M}^{*}} = \underset{1 \leq j \leq N}{\min} d^{\textbf{E}^{*}\textbf{M}^{*}}(x,x_{j})$, then we always have the conditional probability:
    \begin{equation}
    \label{eq:th2}
        \mathbb{P}\left( \frac{d_{max}^{\textbf{E}^{*}\textbf{M}^{*}} - d_{min}^{\textbf{E}^{*}\textbf{M}^{*}}}{d_{min}^{\textbf{E}^{*}\textbf{M}^{*}}} \geq \Theta(\mathcal{D}, \lambda) \middle| \text{$\lambda$ $>$ 0} \right) = 1
    \end{equation}
    where $d^{\textbf{E}^{*}\textbf{M}^{*}}(x,x_{j}) = \frac{\|\textbf{M}^{*}(\textbf{E}^{*}(x)) - \textbf{M}^{*}(\textbf{E}^{*}(x_{i})) \|_{2}}{rank(\textbf{M}^{*})}$, $\mathcal{D}$ denotes the training dataset and $\Theta(\mathcal{D},\lambda)$ depends on the training set and regularization penalty parameter $\lambda$.
\end{theorem}
By examining Eq.(\ref{eq:th2}), it becomes evident that the min-max distance ratio inherently possesses a lower bound due to the incorporation of the nuclear norm regularization term. This suggests that the penalty parameter $\lambda$ governs the establishment of the lower bound. Given that the min-max distance ratio maintains a consistent lower bound, LoRAE is equipped to effectively discriminate between similar and dissimilar data points. Consequently, the learned embeddings capture inherent similarities, leading to enhanced performance on downstream tasks.
\subsection{Rank of Latent Space of LoRAE}
\label{sec:rank_latent_space}
As mentioned in Section~\ref{sec:effect_of_lambda}, it is intuitively understood that stronger regularization in LoRAE leads to a reduction in the rank of the latent space. Additionally, within the same section, we presented empirical evidence of an inverse correlation between the rank of the latent space and the parameter $\lambda$. To provide robust mathematical support, we propose the following proposition.
\begin{proposition}
    \emph{The rank of the latent space follows $\mathcal{O}(1/
    \lambda)$.}
\end{proposition}
Proof of this proposition is deferred to supplementary material.
\section{Conclusion}
A pivotal element within autoencoder methodologies revolves around minimizing the information capacity of the latent space it learns. In this study, we extend beyond mere implicit measures and actively minimize the latent capacity of an autoencoder. Our approach involves minimizing the rank of the empirical covariance matrix associated with the latent space. This is achieved by incorporating a nuclear norm penalty term into the loss function. This addition aids the autoencoder in acquiring representations characterized by significantly reduced dimensions. 

Incorporating a nuclear norm penalty alongside the vanilla reconstruction loss introduces various mathematical assurances. We provide robust mathematical assurances regarding the convergence of our algorithm and elucidate the factors contributing to its commendable performance in downstream tasks. Comparison experiments across multiple domains involving image generation and representation learning indicated that our learning algorithm acquires more reliable feature embedding than baseline methods. Both the theoretical and experimental results clearly demonstrated the necessity/significance of learning low-dimensional embeddings in autoencoders. 
\section*{Acknowledgement}
Alokendu Mazumder is supported by the Prime Minister's Research Fellowship (PMRF), India.

{\small
\bibliographystyle{unsrt}
\bibliography{references}
}
\clearpage
\onecolumn
\setcounter{section}{0}
\setcounter{figure}{0}
\setcounter{table}{0}

\begin{center}
    \LARGE \textbf{Supplementary Material}
\end{center}\

\vspace{2cm}

\section{Experiment Parameters}
\subsection{Dataset}
This paper encompasses a range of experiments conducted using the MNIST and CelebA datasets. To ensure uniformity and facilitate comparisons, all images from the MNIST dataset were resized to dimensions of 32 $\times$ 32 pixels. Likewise, with the CelebA dataset, images were initially center-cropped to 148 $\times$ 148 pixels and subsequently resized to 64 $\times$ 64 pixels.

\subsection{Model Architecture}
The encoder and decoder architectures for each experiment are detailed below. The notation $Conv_{n}$ and $ConvT_{n}$ signify a convolutional and transposed-convolutional layer with an output channel dimension of $n$ respectively. All convolutional layers employ a $4 \times 4$ kernel size with a stride of 2 and padding of 1. $FC_{n}$ denotes a fully connected network with an output dimension of $n$.
\vspace{10pt}
\begin{table}[!htbp]
    \centering
    \caption{Architecture of encoder and decoder for MNIST and CelebA dataset.}
    \begin{tabular}{c | c | c}
        \hline
        \multirow{2}{*}{Datasets} & \multirow{2}{*}{MNIST} & \multirow{2}{*}{CelebA} \\
         &  &  \\
        \hline
        \hline
        Encoder & \begin{tabular}{@{}c@{}}$x \in \mathbb{R}^{32 \times 32 \times 1}$ \\ $\rightarrow Conv_{32} \rightarrow ReLU$ \\ $\rightarrow Conv_{64} \rightarrow ReLU$ \\ $\rightarrow Conv_{128} \rightarrow ReLU$ \\ $\rightarrow Conv_{256} \rightarrow ReLU$ \\ Flatten 1024 \\ $\rightarrow FC_{128} \rightarrow z \in \mathbb{R}^{128}$\end{tabular} & \begin{tabular}{@{}c@{}}$x \in \mathbb{R}^{64 \times 64 \times 3}$ \\ $\rightarrow Conv_{128} \rightarrow ReLU$ \\ $\rightarrow Conv_{256} \rightarrow ReLU$ \\ $\rightarrow Conv_{512} \rightarrow ReLU$ \\ $\rightarrow Conv_{1024} \rightarrow ReLU$ \\ Flatten 16,384 \\ $\rightarrow FC_{256} \rightarrow z \in \mathbb{R}^{256}$\end{tabular} \\
        \hline
        Decoder & \begin{tabular}{@{}c@{}}$z \in \mathbb{R}^{128}$ \\ $FC_{8096}$ \\ Reshape to $8 \times 8 \times 128$ \\ $\rightarrow ConvT_{64} \rightarrow ReLU$ \\ $\rightarrow ConvT_{32} \rightarrow ReLU$ \\ $\rightarrow ConvT_{3} \rightarrow Tanh$ \\ $\hat{x} \in \mathbb{R}^{32 \times 32 \times 1}$\end{tabular} & \begin{tabular}{@{}c@{}}$z \in \mathbb{R}^{512}$ \\ $FC_{65536}$ \\ Reshape to $8 \times 8 \times 1024$ \\ $\rightarrow ConvT_{512} \rightarrow ReLU$ \\ $\rightarrow ConvT_{256} \rightarrow ReLU$ \\ $\rightarrow ConvT_{128} \rightarrow ReLU$ \\ $\rightarrow ConvT_{3} \rightarrow Tanh$ \\ $\hat{x} \in \mathbb{R}^{64 \times 64 \times 3}$\end{tabular} \\
        \hline
    \end{tabular}
    
\end{table}
\newpage
\subsection{Hyperparameter Settings}
Our model underwent training based on the hyperparameter settings provided below.
\begin{table}[!htbp]
\centering
\caption{The hyperparameters for each experiment are elaborated in the following table. The determination of the number of epochs was guided by the aim of attaining a stage of converged reconstruction error.}
\label{tab:resultsC}
\begin{tabular}{c c c}
\hline
Dataset & MNIST & CelebA \\
\hline
\hline
Batch Size & 32 & 32 \\
\hline
Epochs & 50 & 100 \\
\hline
Training Examples & 60,000 & 16,2079 \\
\hline
Test Examples & 10,000 & 20,000 \\
\hline
Dimension of Latent Space & 128 & 128 \\
\hline
Learning Rate & $10^{-3}$ & $10^{-3}$ \\
\hline
$\lambda$ & $10^{-3}$ & $10^{-5}$ \\
\hline
\end{tabular}
\end{table}
\FloatBarrier

\section{Theoretical Analysis}

In this section, we provide a detailed proof for each of the theorems introduced in our paper. Before delving into the proof explanations, we'll establish an understanding of the symbols and terms that will be employed throughout the proofs.

\subsection{Notations}
\begin{enumerate}
\item \emph{For our proposed model:}
\begin{itemize}

    \item We denote the combined parameter of encoder $(\textbf{E})$, decoder $(\textbf{D})$ and the matrix between encoder and decoder $(\textbf{M})$ by $\textbf{w}$. Hence, from now onwards, $\textbf{w}$ is the parameter set of our model.
    
    \item $\textbf{w}_{t}$ denotes the parameter of our model at $t^{th}$ iteration.
    \item $\textbf{w}^{*}$ denotes the parameter of our model after convergence. 

    \item We denote the loss function of our model as:
    \begin{equation}
    \label{eq:reduced_loss}
        \mathcal{L}(\textbf{w}) = \mathcal{L}(\textbf{E}, \textbf{D}, \textbf{M}) = \underbrace{\|\textbf{D} - (\textbf{M}(\textbf{E}(x))\|_{2}^{2}}_{\text{$\mathcal{L}_{mse}(\textbf{w}) = \mathcal{L}_{mse}(\textbf{E}, \textbf{D}, \textbf{M})$}} + \|\textbf{M}\|_{*}
    \end{equation}
    or, in short hand $\mathcal{L}(\textbf{w}) = \mathcal{L}_{mse}(\textbf{w}) + \|\textbf{M}\|_{*}$.
\end{itemize}
    \item \emph{For ADAM Optimizer:}
    \begin{itemize}

        \item The ADAM update for our model can be written as:
        \begin{equation}
        \label{eq:adam_update}
            \textbf{w}_{t+1} = \textbf{w}_{t} - \alpha(V_{t}^{1/2} + diag(\epsilon\mathbb{I}))^{-1}\textbf{m}_{t}
        \end{equation}
        where , $\textbf{m}_{t} = \beta_{1}\textbf{m}_{t-1} + (1 - \beta_{1})\nabla\mathcal{L}(\textbf{w}_{t})$, $\textbf{v}_{t} = \beta_{2}\textbf{v}_{t-1} + (1 - \beta_{2})(\nabla\mathcal{L}(\textbf{w}_{t}))^{2}$, $V_{t} = diag(\textbf{v}_{t})$ is a diagonal matrix, $\beta_{1}$, $\beta_{2}$ $\in (0,1)$ and $\epsilon > 0$. 
        \item $\alpha > 0$, is the constant step size.
        \item One can clearly see from the equation $\textbf{v}_{t} = \beta_{2}\textbf{v}_{t-1} + (1 - \beta_{2})(\nabla\mathcal{L}(\textbf{w}_{t}))^{2}$ that $\textbf{v}_{t}$ will be always non-negative. Also, the term $\epsilon$ in $diag(\epsilon\mathbb{I})$ will always keep the matrix (diagonal matrix) $(V_{t}^{1/2} + diag(\epsilon\mathbb{I}))^{-1}$ positive definite (PD). 
        \item From now onward, to avoid using too much terms in derivation, we will denote the matrix $(V_{t}^{1/2} + diag(\epsilon\mathbb{I}))^{-1}$ as $\textbf{A}_{t}$.
        \item Hence, the ADAM update in Eq.(\ref{eq:adam_update}) will now look like this.
        \begin{equation}
            \textbf{w}_{t+1} = \textbf{w}_{t} - \alpha\textbf{A}_{t}\textbf{m}_{t}
        \end{equation}
        \item We will denote the gradient of the loss function as $\nabla\mathcal{L}
        (\textbf{w})$ for simplicity in rest of our proof.
        
    \end{itemize}
\end{enumerate}
\subsection{Proofs}
\begin{theorem}
Let the loss function $\mathcal{L}(\textbf{E}, \textbf{D}, \textbf{M})$ be $K-$Lipchitz and let $\gamma < \infty$ be an upper bound on the norm of the gradient of $\mathcal{L}$. Then the following holds for the deterministic version (when batch size = total dataset) of Algorithm (1):

For any $\sigma > 0$ if we let $\alpha = \sqrt{2(\mathcal{L}(\textbf{E}_{0}, \textbf{D}_{0}, \textbf{M}_{0}) - \mathcal{L}(\textbf{E}^{*}, \textbf{D}^{*}, \textbf{M}^{*}))/K\delta^{2}T}$, then there exists a natural number $T(\sigma,\delta)$ (depends on $\sigma$ and $\delta$) such that $\|\mathcal{L}(\textbf{E}_{t}, \textbf{D}_{t}, \textbf{M}_{t})\|_{2} \leq \sigma$ for some $t \geq T(\sigma,\delta)$, where $\delta^{2} = \frac{\gamma^{2}}{\epsilon^{2}}$.
\end{theorem}

\begin{proof}
We aim to prove Theorem (1) with contradiction. Let $\|\nabla\mathcal{L}(\textbf{w}_{t})\|_{2} > \sigma > 0$ for all $t \in \{1, 2, \dots\}$. Using Lipchitz continuity, we can write:

\begin{align}
    \mathcal{L}(\textbf{w}_{t+1}) - \mathcal{L}(\textbf{w}_{t})\leq & \hspace{5pt}\nabla\mathcal{L}(\textbf{w}_{t})^{T}(\textbf{w}_{t+1} - \textbf{w}_{t})
    + \frac{K}{2}\|\textbf{w}_{t+1} - \textbf{w}_{t}\|_{2}^{2}\notag \\
    \leq & -\alpha\nabla\mathcal{L}(\textbf{w}_{t})^{T}(\textbf{A}_{t}\textbf{m}_{t}) + \frac{K}{2}\alpha^{2}\|\textbf{A}_{t}\textbf{m}_{t}\|_{2}^{2}\label{eq:1}
\end{align}
One can clearly see that $\textbf{A}_{t}$ is \emph{positive definite} (PD). From here, we will find an upper bound and lower bound on the last and first terms of RHS of Eq.(\ref{eq:1}), respectively. 

Consider the term $\|\textbf{A}_{t}\textbf{m}_{t}\|_{2}$. We have $\lambda_{max}(\textbf{A}_{t}) \leq  \frac{1}{\epsilon + \underset{1 \leq i \leq |\textbf{v}_{t}|}{min} \sqrt{(\textbf{v}_{t})_{i}}}$. Further we note that recursion of $\textbf{v}_{t}$ can be solved as $\textbf{v}_{t} = (1 - \beta_{2})\sum_{j=1}^{t}\beta_{2}^{t-j}(\nabla\mathcal{L}(\textbf{w}_{j}))^{2}$. Now we define $\rho_{t} = \underset{1 \leq j \leq t, 1 \leq k \leq |\textbf{v}_{t}|}{min}(\nabla\mathcal{L}(\textbf{w}_{j})^{2})_{k}$. This gives us the following:
\begin{align}
    \lambda_{max}(\textbf{A}_{t}) \leq \frac{1}{\epsilon + \sqrt{(1-\beta_{2}^{t})\rho_{t}}}
\end{align}
The equation of $\textbf{m}_{t}$ without recursion is $\textbf{m}_{t} = (1 - \beta_{1})\sum_{j=1}^{t}\beta_{1}^{t-j}\nabla\mathcal{L}(\textbf{w}_{j})$. Let us define $\gamma_{t} = \underset{1 \leq j \leq t}{max}\|\nabla\mathcal{L}(\textbf{w}_{j})\|$ then by using triangle inequality, we have $\|\textbf{m}_{t}\|_{2} \leq (1 - \beta_{1}^{t})\gamma_{t}$. We can rewrite $\|\textbf{A}_{t}\textbf{m}_{t}\|_{2}$ as:
\begin{align}
    \|\textbf{A}_{t}\textbf{m}_{t}\|_{2} \leq &\hspace{5pt} \frac{(1-\beta_{1}^{t})\gamma_{t}}{\epsilon + \sqrt{\rho_{t}(1-\beta_{2}^{t})}} \leq \frac{(1-\beta_{1}^{t})\gamma_{t}}{\epsilon} \leq \frac{\gamma_{t}}{\epsilon}\label{eq:3}
\end{align}
Taking $\gamma_{t-1} = \gamma_{t} = \gamma$ and plugging Eq.(\ref{eq:3}) in Eq.(\ref{eq:1}):
\begin{align}
    \mathcal{L}(\textbf{w}_{t+1}) - \mathcal{L}(\textbf{w}_{t})\leq -\alpha\nabla\mathcal{L}(\textbf{w}_{t})^{T}(\textbf{A}_{t}\textbf{m}_{t})\hspace{5pt} + \frac{K}{2}\alpha^{2}\frac{\gamma^{2}}{\epsilon^{2}}\label{eq:4} 
\end{align}
Now, we will investigate the term $\nabla\mathcal{L}(\textbf{w}_{t})^{T}(\textbf{A}_{t}\textbf{m}_{t})$ separately, \emph{i.e.} we will find a lower bound on this term. To analyze this, we define the following sequence of functions: 
\begin{align}
    P_{j} - \beta_{1}P_{j-1} = & \hspace{5pt}\nabla\mathcal{L}(\textbf{w}_{t})^{T}\textbf{A}_{t}(\textbf{m}_{j} - \beta_{1}\textbf{m}_{j-1})\notag \\
    = & \hspace{5pt}(1 - \beta_{1})\nabla\mathcal{L}(\textbf{w}_{t})^{T}(\textbf{A}_{t}\nabla\mathcal{L}(\textbf{w}_{j}))\notag
\end{align}
At $j = t$, we have:
\begin{align}
    P_{t} - \beta_{1}P_{t-1} \geq & \hspace{5pt}(1 - \beta_{1})\|\nabla\mathcal{L}(\textbf{w}_{t})\|_{2}^{2}\lambda_{min}(\textbf{A}_{t})\notag
\end{align}
Let us (again) define $\gamma_{t-1} = \underset{1 \leq j \leq t-1}{max}\|\nabla\mathcal{L}(\textbf{w}_{j})\|_{2}$, and $\forall j \in \{1, 2, \dots t-1\}$:
\begin{align}
    P_{j} - \beta_{1}P_{j-1} \geq & \hspace{5pt}-(1 - \beta_{1})\|\nabla\mathcal{L}(\textbf{w}_{t})\|_{2} \gamma_{t-1}\lambda_{max}(\textbf{A}_{t})\notag
\end{align}
Now, we note the following identity:
\begin{align}
    P_{t} - \beta_{1}^{t}P_{0} = & \hspace{5pt}\sum_{j=1}^{t-1}\beta_{1}^{j}(P_{t-j} - \beta_{1}P_{t-j-1})\notag
\end{align}
Now, we use the lower bounds proven on $P_{j} - \beta_{1}P_{j-1}$ $\forall j \in \{1, 2, \dots t-1\}$ and  $P_{t} - \beta_{1}P_{t-1}$ to lower bound the above sum as:
\begin{align}
    P_{t} - \beta_{1}^{t}P_{0}\geq & \hspace{5pt}(1 - \beta_{1})\|\nabla\mathcal{L}(\textbf{w}_{t})\|_{2}^{2}\lambda_{min}(\textbf{A}_{t}) - (1 - \beta_{1})\|\nabla\mathcal{L}(\textbf{w}_{t})\|_{2} \gamma_{t-1}\lambda_{max}(\textbf{A}_{t})\sum_{j=0}^{t-1}\beta_{1}^{j}\notag\\
    \geq & \hspace{5pt}(1 - \beta_{1})\|\nabla\mathcal{L}(\textbf{w}_{t})\|_{2}^{2}\lambda_{min}(\textbf{A}_{t}) - (\beta_{1} - \beta_{1}^{t})\|\nabla\mathcal{L}(\textbf{w}_{t})\|_{2} \gamma_{t-1}\lambda_{max}(\textbf{A}_{t})\notag\\
    \geq & \hspace{5pt}\|\nabla\mathcal{L}(\textbf{w}_{t})\|_{2}^{2}\left((1 - \beta_{1})\lambda_{min}(\textbf{A}_{t}) - \frac{(\beta_{1} - \beta_{1}^{t})\gamma_{t-1}\lambda_{max}(\textbf{A}_{t})}{\|\nabla\mathcal{L}(\textbf{w}_{t})\|_{2}}\right)\notag\\
    \geq & \hspace{5pt}\|\nabla\mathcal{L}(\textbf{w}_{t})\|_{2}^{2}\left((1 - \beta_{1})\lambda_{min}(\textbf{A}_{t}) - \frac{(\beta_{1} - \beta_{1}^{t})\gamma_{t-1}\lambda_{max}(\textbf{A}_{t})}{\sigma}\right) \hspace{5pt}\label{eq:5}\text{\scriptsize (From Contradiction)}
\end{align}
The inequality in Eq.(\ref{eq:5}) will be maintained as the term $\left((1 - \beta_{1})\lambda_{min}(\textbf{A}_{t}) - \frac{(\beta_{1} - \beta_{1}^{t})\gamma_{t-1}\lambda_{max}(\textbf{A}_{t})}{\sigma}\right)$ is lower bounded by some positive constant $c$. We will show this later in \textbf{Extension 1}. 

Hence, we let $\left((1 - \beta_{1})\lambda_{min}(\textbf{A}_{t}) - \frac{(\beta_{1} - \beta_{1}^{t})\gamma_{t-1}\lambda_{max}(\textbf{A}_{t})}{\sigma}\right) \geq c > 0$ and put $P_{0} =0$ (from definition and initial conditions) in the above equation and get:
\begin{align}
    P_{t} = & \hspace{5pt} \nabla\mathcal{L}(\textbf{w}_{t})^{T}(\textbf{A}_{t}\textbf{m}_{t}) \geq \hspace{5pt} c\|\nabla\mathcal{L}(\textbf{w}_{t})\|_{2}^{2}\label{eq:6} 
\end{align}
Now we are done with computing the bounds on the terms in Eq.(\ref{eq:1}). Hence, we combine Eq.(\ref{eq:6}) with Eq.(\ref{eq:4}) to get:
\begin{align}
    \mathcal{L}(\textbf{w}_{t+1}) - \mathcal{L}(\textbf{w}_{t}) \leq & \hspace{5pt} -\alpha c\|\nabla\mathcal{L}(\textbf{w}_{t})\|_{2}^{2} + \frac{K}{2}\alpha^{2}\frac{\gamma^{2}}{\epsilon^{2}}\notag
\end{align}
Let $\delta^{2} = \frac{\gamma^{2}}{\epsilon^{2}}$ for simplicity. We have:
\begin{align}
    \mathcal{L}(\textbf{w}_{t+1}) - \mathcal{L}(\textbf{w}_{t})\leq & \hspace{5pt} -\alpha c\|\nabla\mathcal{L}(\textbf{w}_{t})\|_{2}^{2} + \frac{K}{2}\alpha^{2}\delta^{2}\notag\\
    \alpha c\|\nabla\mathcal{L}(\textbf{w}_{t})\|_{2}^{2} \leq & \hspace{5pt} \mathcal{L}(\textbf{w}_{t}) - \mathcal{L}(\textbf{w}_{t+1}) + \frac{K}{2}\alpha^{2}\delta^{2}\notag\\
    \|\nabla\mathcal{L}(\textbf{w}_{t})\|_{2}^{2} \leq & \hspace{5pt} \frac{\mathcal{L}(\textbf{w}_{t}) - \mathcal{L}(\textbf{w}_{t+1})}{\alpha c} + \frac{K\alpha\delta^{2}}{2c}\label{eq:7}
\end{align}    
From Eq.(\ref{eq:7}), we have the following inequalities:
\begin{align}
\left\{
\begin{aligned}
    \|\nabla\mathcal{L}(\textbf{w}_{0})\|_{2}^{2} \leq & \hspace{5pt} \frac{\mathcal{L}(\textbf{w}_{0}) - \mathcal{L}(\textbf{w}_{1})}{\alpha c} + \frac{K\alpha\delta^{2}}{2c} & \notag\\
    \|\nabla\mathcal{L}(\textbf{w}_{1})\|_{2}^{2} \leq & \hspace{5pt} \frac{\mathcal{L}(\textbf{w}_{1}) - \mathcal{L}(\textbf{W}_{2})}{\alpha c} + \frac{K\alpha\delta^{2}}{2c} & \notag\\
    & \vdots \\
    \|\nabla\mathcal{L}(\textbf{w}_{T-1})\|_{2}^{2} \leq & \hspace{5pt} \frac{\mathcal{L}(\textbf{w}_{T-1}) - \mathcal{L}(\textbf{w}_{t})}{\alpha c} + \frac{K\alpha\delta^{2}}{2c} &\\
\end{aligned}
\right.
\end{align}
Summing up all the inequalities presented above, we obtain:
\begin{align*}
    \sum_{t=0}^{T-1} \|\nabla\mathcal{L}(\textbf{w}_{t})\|_{2}^{2} \leq & \hspace{5pt} \frac{\mathcal{L}(\textbf{w}_{0}) - \mathcal{L}(\textbf{w}_{t})}{\alpha c} + \frac{K\alpha\delta^{2}T}{2c}
\end{align*}
The inequality remains valid if we substitute $\|\nabla\mathcal{L}(\textbf{w}_{t})\|_{2}^{2}$ with $\underset{0\leq t \leq T-1}{min}\|\nabla\mathcal{L}(\textbf{w}_{t})\|_{2}^{2}$ within the summation on the left-hand side (LHS).
\begin{align}
    \underset{0 \leq t \leq T-1}{min}\|\nabla\mathcal{L}(\textbf{w}_{t})\|_{2}^{2}T \leq & \hspace{5pt} \frac{\mathcal{L}(\textbf{w}_{0}) - \mathcal{L}(\textbf{w}^{*})}{\alpha c} + \frac{K\alpha\delta^{2}T}{2c}\notag\\
    \underset{0 \leq t \leq T-1}{min}\|\nabla\mathcal{L}(\textbf{w}_{t})\|_{2}^{2} \leq & \hspace{5pt} \frac{\mathcal{L}(\textbf{w}_{0}) - \mathcal{L}(\textbf{w}^{*})}{\alpha c T} + \frac{K\alpha\delta^{2}}{2c}\notag\\
    \underset{0 \leq t \leq T-1}{min}\|\nabla\mathcal{L}(\textbf{w}_{t})\|_{2}^{2} \leq & \hspace{1pt} \frac{1}{\sqrt{T}}\left(\frac{\mathcal{L}(\textbf{w}_{0}) - \mathcal{L}(\textbf{w}^{*})}{cb} + \frac{K\delta^{2}b}{2c}\right)\notag
\end{align}
where $b = \alpha\sqrt{T}$. We set $b = \sqrt{2(\mathcal{L}(\textbf{w}_{0}) - \mathcal{L}(\textbf{w}^{*})\delta^{2})/K\delta^{2}}$, and we have:
\begin{align}
    \underset{0 \leq t \leq T-1}{\min} \|\nabla\mathcal{L}(\textbf{w}_{t})\|_{2} \leq \left(\frac{2K\delta^{2}}{T}(\mathcal{L}(\textbf{w}_{0})-\mathcal{L}(\textbf{w}^{*}))\right)^{\frac{1}{4}}\notag
\end{align}
When $T \geq \left(\frac{2K\delta^{2}}{\sigma^{4}}(\mathcal{L}(\textbf{w}_{0})-\mathcal{L}(\textbf{w}^{*}))\right) = T(\sigma)$, we will have $\underset{0 \leq t \leq T-1}{\min} \|\nabla\mathcal{L}(\textbf{w}_{t})\|_{2} \leq \sigma$ which will contradict the assumption, \emph{i.e.} ($\|\nabla\mathcal{L}(\textbf{w}_{t})\|_{2} > \sigma$ for all $t \in \{1, 2, \dots\}$). Hence, completing the proof.
\end{proof}

\begin{theorem}
    Given any set of i.i.d x, $x_{1}, x_{2}, \dots, x_{N}$ $\in \mathbb{R}^{l}$, we denote $d_{max}^{\textbf{E}^{*}\textbf{M}^{*}} = \underset{1 \leq j \leq N}{\max} d^{\textbf{E}^{*}\textbf{M}^{*}}(x,x_{j})$ and \\$d_{min}^{\textbf{E}^{*}\textbf{M}^{*}} = \underset{1 \leq j \leq N}{\min} d^{\textbf{E}^{*}\textbf{M}^{*}}(x,x_{j})$, then we always have the conditional probability:
    \begin{equation}
        \mathbb{P}\left( \frac{d_{max}^{\textbf{E}^{*}\textbf{M}^{*}} - d_{min}^{\textbf{E}^{*}\textbf{M}^{*}}}{d_{min}^{\textbf{E}^{*}\textbf{M}^{*}}} \geq \Theta(\mathcal{D}, \lambda) \middle| \text{$\lambda$ $>$ 0} \right) = 1
    \end{equation}
    where $d^{\textbf{E}^{*}\textbf{M}^{*}}(x,x_{j}) = \frac{\|\textbf{M}^{*}(\textbf{E}^{*}(x)) - \textbf{M}^{*}(\textbf{E}^{*}(x_{i})) \|_{2}}{rank(\textbf{M}^{*})}$, $\mathcal{D}$ denotes the training dataset and $\Theta(\mathcal{D},\lambda)$ depends on the training set and regularization penalty parameter $\lambda$.
\end{theorem}
\begin{proof}
    As $\textbf{w}^{*}$ is learned from Algorithm (1), we always have:
    \begin{align}
        \mathcal{L}(\textbf{w}^{*}) \quad \leq & \quad \mathcal{L}(\textbf{w}_{0})\notag
    \end{align}
    where, $\mathcal{L}(\textbf{w}_{0})$ is loss of our model at $0^{th}$ epoch. Hence,
    \begin{align}
        \mathcal{L}_{mse}(\textbf{w}^{*}) + \lambda\|\textbf{M}^{*}\|_{*} \quad \leq & \quad \mathcal{L}_{mse}(\textbf{W}_{0}) + \lambda\|\textbf{M}_{0}\|_{*}\notag\\
        \lambda\|\textbf{M}^{*}\|_{*} \quad \leq & \quad \mathcal{L}_{mse}(\textbf{w}_{0}) -  \mathcal{L}_{mse}(\textbf{w}^{*}) + \lambda\|\textbf{M}_{0}\|_{*}\notag\\
        \|\textbf{M}^{*}\|_{*} \quad \leq & \quad \frac{1}{\lambda}\left(\mathcal{L}_{mse}(\textbf{w}_{0}) -  \mathcal{L}_{mse}(\textbf{w}^{*})\right) + \|\textbf{M}_{0}\|_{*}\notag\\
        \|\textbf{M}^{*}\|_{*} \quad \leq & \quad \frac{1}{\lambda}\left(c_{1} -  c_{2}\right) + c_{3}\label{eq:12}
\end{align}
where, $c_{1} = \mathcal{L}_{mse}(\textbf{w}_{0})$, $c_{2} = \mathcal{L}_{mse}(\textbf{w}^{*})$, and $c_{3} = \|\textbf{M}_{0}\|_{*}$.
Now, from Eq.(\ref{eq:12}) we can estimate an upperbound on the rank of matrix $\textbf{M}^{*}$:
\begin{align}
    rank(\textbf{M}^{*}) \quad \leq & \quad c\left(\frac{1}{\lambda}\left(c_{1} -  c_{2}\right) + c_{3}\right) \quad\text{(where $c \in \mathbb{R}^{+}$)}\label{eq:13}
\end{align}
Using the definition of $d_{max}^{\textbf{E}^{*}\textbf{M}^{*}}$ and $d_{min}^{\textbf{E}^{*}\textbf{M}^{*}}$, we have:
\begin{align}
    \frac{d_{max}^{\textbf{E}^{*}\textbf{M}^{*}} - d_{min}^{\textbf{E}^{*}\textbf{M}^{*}}}{d_{min}^{\textbf{E}^{*}\textbf{M}^{*}}} \quad = & \quad \frac{\underset{i \in [n]}{max}\frac{\|\textbf{M}^{*}(\textbf{E}^{*}(x)) - \textbf{M}^{*}(\textbf{E}^{*}(x_{i})) \|_{2}}{rank(\textbf{M}^{*})} - \underset{i \in [n]}{min}\frac{\|\textbf{M}^{*}(\textbf{E}^{*}(x)) - \textbf{M}^{*}(\textbf{E}^{*}(x_{i})) \|_{2}}{rank(\textbf{M}^{*})}}{\underset{i \in [n]}{min}\frac{\|\textbf{M}^{*}(\textbf{E}^{*}(x)) - \textbf{M}^{*}(\textbf{E}^{*}(x_{i})) \|_{2}}{rank(\textbf{M}^{*})}}\notag\\
    \quad = & \quad \frac{\underset{i \in [n]}{max}\frac{\|\textbf{M}^{*}(\textbf{E}^{*}(x)) - \textbf{M}^{*}(\textbf{E}^{*}(x_{i})) \|_{2}}{rank(\textbf{M}^{*})}}{\underset{i \in [n]}{min}\frac{\|\textbf{M}^{*}(\textbf{E}^{*}(x)) - \textbf{M}^{*}(\textbf{E}^{*}(x_{i})) \|_{2}}{rank(\textbf{M}^{*})}} - 1\notag\\
    \quad \geq & \quad \frac{\underset{i \in [n]}{max}\frac{\|\textbf{M}^{*}(\textbf{E}^{*}(x)) - \textbf{M}^{*}(\textbf{E}^{*}(x_{i})) \|_{2}}{c\left(\frac{1}{\lambda}\left(c_{1} -  c_{2}\right) + c_{3}\right)}}{\underset{i \in [n]}{min}\frac{\|\textbf{M}^{*}(\textbf{E}^{*}(x)) - \textbf{M}^{*}(\textbf{E}^{*}(x_{i})) \|_{2}}{rank(\textbf{M}^{*})}} - 1 \quad \text{(Using Eq.(15))}\notag\\
    \quad \geq & \quad \frac{L(\mathcal{D})}{c\left(\frac{1}{\lambda}\left(c_{1} -  c_{2}\right) + c_{3}\right)} \quad \text{$\left(here, L(\mathcal{D}) = \frac{\underset{i \in [n]}{max}\|\textbf{M}^{*}(\textbf{E}^{*}(x)) - \textbf{M}^{*}(\textbf{E}^{*}(x_{i})) \|_{2}}{\underset{i \in [n]}{min}\frac{\|\textbf{M}^{*}(\textbf{E}^{*}(x)) - \textbf{M}^{*}(\textbf{E}^{*}(x_{i})) \|_{2}}{rank(\textbf{M}^{*})}}\right)$}\notag\\
    \quad \geq & \quad \frac{\lambda L(\mathcal{D})}{c(c_{1} - c_{2}) + \lambda cc_{3}}\notag\\
    \quad \geq & \quad \frac{\lambda L(\mathcal{D})}{cc_{1}} = \Theta(\lambda, \mathcal{D}) > 0 \quad \text{when $\lambda > 0$}\notag
\end{align}
hence, completing the proof.
\end{proof}
\begin{proposition}
The rank of the latent space follows $\mathcal{O}(1/\lambda)$.
\end{proposition}
\begin{proof}
    Let $\textbf{E}^{*}$ denote the trained encoder of our model and let $x \in \mathbb{R}^{m \times n \times c}$ be an image with dimension $m \times n$ and $c$ number of channels. Let $y = \textbf{E}^{*}(x)$, then we can define the latent space of our model (LoRAE) as:
    \begin{align}
        z = \textbf{M}^{*}y = \textbf{M}^{*}(\textbf{E}^{*}(x))
    \end{align}
We define the rank of the latent space as the number of non-zero singular values of the covariance matrix of latent space, i.e $\mathbb{E}_{\mathcal{D}}[zz^{T}]$. We can write:
\begin{align}
    \mathbb{E}_{\mathcal{D}}[zz^{T}] \quad = & \quad \mathbb{E}_{\mathcal{D}}[\textbf{M}^{*}yy^{T}\textbf{M}^{*T}]\label{eq:15}
\end{align}
Eq.(\ref{eq:13}) from Theorem 2 states that:
\begin{align*}
    rank(\textbf{M}^{*}) \quad \leq & \quad \frac{c}{\lambda}\left(c_{1} -  c_{2}\right) + cc_{3} \quad\text{(where $c \in \mathbb{R}^{+}$)}
\end{align*}
As $\textbf{M}^{*}$ is deterministic in Eq.(\ref{eq:15}), the covariance matrix can be re-written as $ \textbf{M}^{*} \mathbb{E}_{\mathcal{D}} \left[ y y^T \right] \textbf{M}^{*T}$. An upper bound on the rank of $ \textbf{M}^{*} \mathbb{E}_{\mathcal{D}} \left[ x x^T \right] \textbf{M}^{*T}$ is the upper bound on the rank of $\textbf{M}^{*}$. Thus, from Eq.(\ref{eq:13}) of Theorem 2, this analysis gives an upper bound on the rank of latent space as $\mathcal{O}(1/\lambda)$.

\end{proof}
\newpage
\begin{extension}
    The term $\left((1 - \beta_{1})\lambda_{min}(\textbf{A}_{t}) - \frac{(\beta_{1} - \beta_{1}^{t})\gamma_{t-1}\lambda_{max}(\textbf{A}_{t})}{\sigma}\right)$ from Eq.(\ref{eq:5}) is always non-negetive.
\end{extension}
\begin{proof}
    We can construct a lower bound on $\lambda_{min}(\textbf{A}_{t})$ and an upper bound on $\lambda_{min}(\textbf{A}_{t})$ as follows:
\begin{align}
        \lambda_{min}(\textbf{A}_{t}) \quad \geq & \quad \frac{1}{\epsilon + \sqrt{\underset{1 \leq j \leq |\textbf{v}_{t}|}{max}(\textbf{v}_{t})_{j}}}\\
        \lambda_{max}(\textbf{A}_{t}) \quad \leq & \quad \frac{1}{\epsilon + \sqrt{\underset{1 \leq j \leq |\textbf{v}_{t}|}{min}(\textbf{v}_{t})_{j}}}
\end{align}
We remember that $\textbf{v}_{t}$ can be rewritten as $\textbf{v}_{t} = \beta_{2}\textbf{v}_{t-1} + (1 - \beta_{2})(\nabla\mathcal{L}(\textbf{w}_{t}))^{2}$, solving this recursion and defining $\rho_{t} = \underset{1 \leq j \leq t, 1 \leq k \leq |\textbf{v}_{t}|}{min}(\nabla\mathcal{L}(\textbf{w}_{j})^{2})_{k}$ and taking $\gamma_{t-1} = \gamma_{t} = \gamma$ we have:
\begin{align*}
    \lambda_{min}(\textbf{A}_{t}) \quad \geq & \quad \frac{1}{\epsilon + \sqrt{(1 - \beta_{2}^{t})\gamma^{2}}}\\
    \lambda_{max}(\textbf{A}_{t}) \quad \leq & \quad \frac{1}{\epsilon + \sqrt{(1 - \beta_{2}^{t})\rho_{t}}}
\end{align*}
Where, $\gamma_{t-1} = \underset{1 \leq j \leq t-1}{max}\|\nabla\mathcal{L}(\textbf{w}_{j})\|_{2}$, and $\forall j \in \{1, 2, \dots t-1\}$. Setting $\rho_{t} = 0$, we can rewrite the term \\$\left((1 - \beta_{1})\lambda_{min}(\textbf{A}_{t}) - \frac{(\beta_{1} - \beta_{1}^{t})\gamma_{t-1}\lambda_{max}(\textbf{A}_{t})}{\sigma}\right)$ as:

\begin{align}
    \left((1 - \beta_{1})\lambda_{min}(\textbf{A}_{t}) - \frac{(\beta_{1} - \beta_{1}^{t})\gamma_{t-1}\lambda_{max}(\textbf{A}_{t})}{\sigma}\right) \geq & \hspace{5pt}\left(\frac{(1 - \beta_{1})}{\epsilon +\gamma\sqrt{(1 - \beta_{2}^{t})}} - \frac{(\beta_{1} - \beta_{1}^{t})\gamma}{\epsilon\sigma}\right)\label{eq:8}\\
    \geq & \hspace{5pt}\frac{\epsilon\sigma(1-\beta_{1}) - \gamma(\beta_{1} - \beta_{1}^{t})(\epsilon + \gamma\sqrt{(1 - \beta_{2}^{t})})}{\epsilon\sigma(\epsilon + \gamma\sqrt{(1 - \beta_{2}^{t})})}\notag\\
    \geq & \hspace{5pt}\gamma(\beta_{1} - \beta_{1}^{t})\frac{\epsilon\left(\frac{\sigma(1 - \beta_{1})}{\gamma(\beta_{1} - \beta_{1}^{t})} - 1\right) - \gamma\sqrt{(1 - \beta_{2}^{t})}}{\epsilon\sigma(\epsilon + \gamma\sqrt{(1 - \beta_{2}^{t})})}\notag\\
    \geq & \hspace{5pt}\gamma(\beta_{1} - \beta_{1}^{t})\left(\frac{\sigma(1 - \beta_{1})}{\gamma(\beta_{1} - \beta_{1}^{t})} - 1\right)\frac{\epsilon - \left(\frac{\gamma\sqrt{(1 - \beta_{2}^{t})}}{\frac{(1 - \beta_{1}\sigma)}{(\beta_{1} - \beta_{1}^{t})\gamma} - 1}\right)}{\epsilon\sigma(\epsilon + \gamma\sqrt{(1 - \beta_{2}^{t})})}\notag
\end{align}

By definition $\beta_{1} \in (0,1)$ and hence $(\beta_{1} - \beta_{1}^{t}) \in (0,\beta_{1})$. This implies that $\frac{(1 - \beta_{1})\sigma}{(\beta_{1} - \beta_{1}^{t})\gamma} > \frac{(1 - \beta_{1})\sigma}{\beta_{1}\gamma} > 1$ where the last inequality follows due to the choice of $\sigma$ as stated in the beginning of this theorem. This allows us to define a constant $\frac{(1 - \beta_{1})\sigma}{\beta_{1}\gamma} - 1 := \psi_{1} > 0$ such that $\frac{(1 - \beta_{1})\sigma}{(\beta_{1} - \beta_{1}^{t})\gamma} - 1 > \psi_{1}$. Similarly, our definition of delta allows us to define another constant $\psi_{2} > 0$ to get:
\begin{align}
    \left(\frac{\gamma\sqrt{(1 - \beta_{2}^{t})}}{\frac{(1 - \beta_{1}\sigma)}{(\beta_{1} - \beta_{1}^{t})\gamma} - 1}\right) \quad < & \quad \frac{\gamma}{\psi_{1}} = \epsilon - \psi_{2}\label{eq:9}
\end{align}
Putting Eq.(\ref{eq:9}) in Eq.(\ref{eq:8}), we get:
\begin{align}
    &\left((1 - \beta_{1})\lambda_{min}(\textbf{A}_{t}) - \frac{(\beta_{1} - \beta_{1}^{t})\gamma_{t-1}\lambda_{max}(\textbf{A}_{t})}{\sigma}\right)\geq & \hspace{5pt}\left(\frac{\gamma(\beta_{1} - \beta_{1}^{2})\psi_{1}\psi_{2}}{\epsilon\sigma(\epsilon + \sigma)}\right) = c > 0\notag
\end{align}

\end{proof}


\end{document}